\documentclass[runningheads]{llncs}
\usepackage[T1]{fontenc}
\usepackage{graphicx}
\usepackage[hidelinks]{hyperref}

\usepackage[caption=false]{subfig}

\usepackage{amsmath}
\usepackage{amssymb}
\usepackage{mathtools}
\usepackage{bm}
\usepackage{booktabs}

\usepackage{multirow}

\DeclareMathOperator*{\argmin}{arg\,min}

\begin{document}
\title{Gradient Boosting Mapping for Dimensionality Reduction and Feature Extraction}
\titlerunning{Gradient Boosting Mapping}
\author{Anri Patron
\and
Ayush Prasad \and
Hoang Phuc Hau Luu
\and
Kai Puolamäki}
\authorrunning{A. Patron et al.}

\institute{University of Helsinki, Helsinki, Finland\\
\email{firstname.lastname@helsinki.fi}}

\maketitle

\begin{abstract}
A fundamental problem in supervised learning is to find \emph{a good set of features} or \emph{distance measures}. If the new set of features is of lower dimensionality and can be obtained by a simple transformation of the original data, they can make the model understandable, reduce overfitting, and even help to detect distribution drift.
We propose a supervised dimensionality reduction method {\em Gradient Boosting Mapping} ({\sc gbmap}), where the outputs of weak learners---defined as one-layer perceptrons---define the embedding. We show that the embedding coordinates provide better features for the supervised learning task, making simple linear models competitive with the state-of-the-art regressors and classifiers. We also use the embedding to find a principled distance measure between points. The features and distance measures automatically ignore directions irrelevant to the supervised learning task. We also show that we can reliably detect out-of-distribution data points with potentially large regression or classification errors. {\sc gbmap} is fast and works in seconds for dataset of million data points or hundreds of features. As a bonus, {\sc gbmap} provides a regression and classification performance comparable to the state-of-the-art supervised learning methods.
\end{abstract}



\section{Introduction}\label{sec:intro}

In supervised learning, the fundamental task is to learn a regressor or classifier function from a set of features---or distances between the training data points---to the target variable, a real number for regression or a discrete value for classification. Several algorithms for finding such regression/classifier functions have been developed \cite{hastie2009,JMLR:v15:delgado14a}, and the field is already quite mature, at least for ``classical'' regression and classification datasets with up to thousands of features. It has even been argued the improvements in supervised learning technology during the past decades have been marginal at best, and much of the progress has been illusory in the sense that, in real-world applications, the marginal gains may be swamped by other sources of uncertainty, such as concept drift, not considered in the classical supervised learning 
paradigm \cite{hand2006Classifier}. 

However, if the data features are selected carefully, even a simple model---such as linear or logistic regression---may perform competitively to the state-of-the-art supervised learning models. The simple model can be more understandable if the new features are obtained by a relatively simple transformation from the original data. Indeed, well-performing supervised learning models are often {\em black boxes} whose working principles are difficult or impossible to understand, which is problematic in many practical applications \cite{guidotti:2018:a}. 

While there are established methods to estimate the regression or classification errors for data from the same distribution used in the training, there are fewer ways to do so for out-of-distribution data in the presence of {\em concept drift} \cite{gama2014survey}. In other words, it is difficult to estimate the {\em uncertainty of the predictions}, which again, for obvious reasons, may hinder the usability of the supervised learning models in real-world applications where we do not have the ground truth target variables available. 

Our intuition is that in real-world regression or classification problems, only a subset of {\em local directions} in the feature space is relevant for the regression or classification task. We call a direction relevant if it affects the regressor/classifier output. A good feature or distance measure should depend only on the relevant directions in the feature space, and they should be insensitive, e.g., to random nuisance features. We postulate that such features are better than the original features for the supervised learning task and may allow the use of simpler models with competitive performance.
Also, to detect prediction uncertainty due to concept drift, we postulate that only the distances to these relevant directions must be considered.

To this end, we introduce {\em Gradient Boosting Mapping} ({\sc gbmap}), a supervised dimensional reduction method based on the gradient-boosting idea that a complicated function can be learned by a sequence of simpler functions known as weak learners. The weak learners are incorporated sequentially so that at each step, the current weak learner tries to correct the errors (residuals) made by the model learned during the earlier steps. In gradient boosting, the family of weak learners should be specified in advance (possibly with some parameterized form) to pick out the best weak learner using an optimization procedure. The weak learners should have low variance and high bias, meaning that the model family of the weak learners should be sufficiently simple. The boosting procedure then aggregates these simple weak learners to create a strong learner who can capture a highly complicated supervised learning model while being robust enough to avoid overfitting. The most common family of weak learners is decision trees, where the corresponding gradient boosting scheme is eXtreme Gradient Boosting ({\sc xgboost}) \cite{chen2016xgboost}, providing piece-wise constant weak learners that converge to a constant value outside the training data distribution. As a side product of the boosting formulation, {\sc gbmap} can also be used as a regression/classification model whose performance is competitive with state-of-the-art models such as {\sc xgboost}.

{\em Contributions} of this paper are: (i) we define a boosting-based embedding method {\sc gbmap} and show its theoretical properties (Sect. \ref{sec:theory}), (ii) we show the embedding is fast to compute (e.g., the runtime on a dataset with a million data points of a dimension of $25$ is around 15 seconds (Sect. \ref{sec:scaling}), and that it performs comparatively to the state-of-the-art regression/classification models (Sect. \ref{sec:regression}) although regression/classification performance is not our primitive objective, (iii) show that the embedding features and the induced distance can be used as features for classification/regression (Sect. \ref{sec:features}), and that (iv) {\sc gbmap} can detect outlier data points with potentially large regression/classification errors without seeing the ground truth target values (Sect. \ref{sec:ood}). Finally, we conclude the paper with a discussion of potential other applications for {\sc gbmap} (Sect. \ref{sec:other}) and a future outlook (Sect. \ref{sec:conclusion}). We included supplementary material in the appendix that is not necessary to understand the contributions of the paper but gives, e.g., additional experimental results that complement those in the paper (Sects. \ref{sec:linear}--\ref{sec:vis2}).

\section{Related Work}\label{sec:related}

The standard methods for dimensionality reduction/data visualization in machine learning are unsupervised methods such as Principal Component Analysis ({\sc pca}) \cite{wold1987Principal} or t-distributed Stochastic Neighbor Embedding ({t-sne}) \cite{vandenmaaten2008} that do not consider the target variable at all; the advantage is simplicity and that these methods can also be used for unlabeled data, but the obvious drawback is that if most features are irrelevant for the supervised learning task, the embedding may not be helpful. For this reason, supervised methods try to take the target variable into account. These methods include Linear Optimal Low rank projection ({\sc lol}) \cite{vogelstein2021Supervised} and {\sc ivis} \cite{szubert2019Structurepreserving}.
{\sc lol} works by calculating the mean of each class and the difference between these means. It then computes the class-centered covariance matrix, using the top eigenvectors and the mean difference to construct the embeddings. {\sc ivis} is a parametric method that uses a Siamese neural network architecture with a triplet loss function to preserve distances of data points in the embedding space.

The term {\em concept drift} means phenomena where the data distribution changes over time; for surveys, see \cite{gama2014survey,zliobaite2016Overview,lu2018Learning}. Most concept drift detection methods need access to the ground truth target variables, especially those of test data points. However, in some real-world applications, we can only have access to the training labels, while the test labels are missing \cite{oikarinen2020Supervised}, and we have to evaluate drift based solely on the model prediction, training labels, and the covariate structures. Methods that detect drift merely based on covariate distribution shifts \cite{shao2014Prototypebased,qahtan2015PCABased} suffer from irrelevant features, resulting in a high false alarm rate \cite{sethi2017reliable}. 
Therefore, taking into account the target variable (training labels, model prediction) is highly important. The literature is limited in this line of work and primarily focuses on concept drift in classification \cite{sethi2017reliable,lindstrom2013Drift}. Recently, \cite{oikarinen2020Supervised} proposed a drift detection algorithm for {\em time series} regression problems. In this work, we demonstrate an added benefit of {\sc gbmap} that its embedding distance can be used internally for drift detection, with high reliability for regression and classification tasks. It is worth noting that this is an efficient built-in concept drift detection of {\sc gbmap}, eliminating the need for an external package to detect drift in {\sc gbmap}.

For the boosting literature, we refer to a textbook such as  \cite{hastie2009}, where we have borrowed the notation used in this paper. Our ideas are mostly related to gradient boosting, where at each boosting iteration, the weak learner is updated to minimize the loss induced by the current strong learner \cite{friedman2001greedy}, which is the accumulation of all previous weak learners. This work uses a family of perceptrons as weak learners to extract simple softplus-like functions from a complicated target function sequentially. These simple functions serve as interpretable feature extractors, as explained in Sect. \ref{sec:explainability}. A related method is {\sc xgboost} \cite{chen2016xgboost}, which uses tree-based regressors and classifiers as weak learners; see Sect. \ref{sec:relboost} for a more detailed discussion.

\section{Theory and Methodology}\label{sec:theory}

\subsection{Definition of Our Model}\label{sec:def}

Assume we have $n$ {\em training data} points $({\bf x}_1,y_1),\ldots,({\bf x}_n,y_n)$ drawn independently from a fixed but usually unknown distribution $({\bf x}_i,y_i)\sim F$, where the target variable is $y_i\in{\mathbb{R}}$ for regression, $y_i\in\{-1,+1\}$ for binary classification, the covariate is ${\bf x}_i\in{\mathbb{R}}^p$, and $i\in[n]=\{1,\ldots,n\}$. ${\bf X}\in{\mathbb{R}}^{n\times p}$ denotes the data matrix where the rows correspond to observations, ${\bf X}_{i\cdot}={\bf x}_i^\intercal$. In this work, we want to predict the value of $y$ given ${\bf x}$ by fitting a model $f:{\mathbb{R}}^p\to{\mathbb{R}}$ to the training data. For regression problems, the prediction $\hat y$ is then given directly by $f$ as $\hat y=f({\bf x})$. For classification problems, the probability of $+1$ is given by $\hat p(y=+1\mid{\bf x})=\sigma(f({\bf x}))$, where $\sigma(z)=1/(1+e^{-z})$ is the sigmoid function (see Sect. \ref{sec:linear}), and the predicted class by the sign as $\hat y={\rm{sign}}\left(f({\bf x})\right)$.

We define $f$ to be an ensemble of $m$ weak learners $f_j:{\mathbb{R}}^p\to{\mathbb{R}}$, where $j\in[m]$, and an initial model $f_0({\bf x})$, which be any function; we use by default $f_0({\bf x})=0$. Other choices of $f_0$ are possible depending on the application, as discussed in Sect. \ref{sec:f0}. $f$ is given by
\begin{equation}\label{eq:f}
f({\bf x})=\sum\nolimits_{j=0}^m{f_j({\bf x})}.
\end{equation}

Here---even though we end up with well-performing regressors and classifiers---our main motivation is not to make the best regressor or classifier but instead use the weak learners to find a lower-dimensional embedding. For this reason, we define the weak learners to be linear projections with non-linearity,
\begin{equation}\label{eq:fj}
f_j({\bf x})=a_j+b_jg({\bf w}_j^\intercal{\bf x}),
\end{equation}
where the intercept and slope terms $a_j\in{\mathbb{R}}$, $b_j\in\{-1,+1\}$, and the projection vectors ${\bf w}_j\in{\mathbb{R}}^p$ are learned from data, as described later in Sect. \ref{sec:learn}. We assume that ${\bf x}$ contains an intercept term (a feature which equals unity) if necessary. The nonlinearity is given by the function $g:{\mathbb{R}}\to{\mathbb{R}}$. In this work, we use softplus $g_{soft+}^\beta(z)=\log{\left(1+e^{\beta z}\right)/\beta}$, the smooth variant (for more efficient optimization) of the ReLU $\lim\nolimits_{\beta\to\infty}{g_{soft+}^\beta(z)}=\max{(0,z)}$.

Formally, we want to find $f$ such that the generalization error defined by
$
L=E_{({\bf{x}},y)\sim F}\left[
l(y,f({\bf x}))
\right]
$
is minimized, where $l:{\mathbb{R}}\times{\mathbb{R}}\to{\mathbb{R}}_{\ge 0}$ is a predefined loss function. We use the quadratic loss $l_{quardatic}(y,y')=(y-y')^2$ for regression problems and logistic loss $l_{logistic}(y,y')=\log{\left(1+e^{-yy'}\right)}$ for classification problems. The logistic loss (used by logistic regression) can be viewed as a smooth variant (again, for easier optimization) of the hinge loss $l_{hinge}(y, y')=\max{(0,1-yy')}$ (used, e.g., by support vector classifiers).

\subsection{Learning the Model Parameters}\label{sec:learn}

The problem can be cast into $2^m$ continuous non-convex optimization problems ($2^m$ possibilities of $b_j$'s), which is NP-hard. The boosting scheme offers a greedy stage-wise search approach to explore this combinatorial state space. We can also view Eq. \eqref{eq:f} as a discretized ensemble of $m$ weak learners, 1-layer perceptrons. 

Given the initial model $f_0$ (by default, $f_0({\bf x})=0$), loss function $l$, non-linearity $g$, and the training data, we can use the boosting ideas to find our model's parameters $a_j$, $b_j$, and ${\bf{w}}_j$. The boosting consists of $m$ iterations. At $j$th iteration, we find the model parameters by solving the following optimization problem:
\begin{equation}\label{eq:argmin}
a_j,b_j,{\bf w}_j=
\argmin\nolimits_{a_j,b_j,{\bf w}_j}{\left(
{\cal L}_j+{\cal R}_j\right)},
\end{equation}
where the empirical loss is
\begin{equation}\label{eq:empirical}
{\cal L}_j=n^{-1}\sum\nolimits_{i=1}^n{
l\left(y_i,
\sum\nolimits_{k=0}^j{f_k({\bf x}_i)}
\right)
}.
\end{equation}
Recall that $f_j({\bf x})$ is parameterized by $a_j$, $b_j$, and ${\bf{w}}_j$. Here, we use a Ridge regularization for numerical stability, given by
$
{\cal R}_j=
\lambda\sum\nolimits_{k=1}^p{{\bf w}_{jk}^2/p}$.
We use the limited-memory Broyden–Fletcher–Goldfarb–Shanno algorithm (LBFGS) implementation in Python {\sc jaxopt} library \cite{jaxopt_implicit_diff} to solve the optimization task in Eq. \eqref{eq:argmin} for $b_j=-1$ and $b_j=+1$ separately, choosing the value of $b_j$ that leads to the smallest loss. We find all of the functions $f_j$ by repeating the optimization of Eq. (3) for all values of $j$ starting from $1$ and ending in $m$.

\subsection{Embedding of Data Points}\label{sec:emb}

The most interesting advantage of our proposed method is that it can be used for supervised dimensionality reduction for various purposes. The {\em embedding} $\phi:{\mathbb{R}}^p\to{\mathbb{R}}^m$ of a data point in ${\bf x}$ is given simply by
\begin{equation}\label{eq:emb}
\phi({\bf x})=\left(f_1({\bf x}),\ldots,f_m({\bf x})\right)^\intercal.
\end{equation}
We define the {\em path distance} between data points ${\bf x},{\bf x}'\in{\mathbb{R}}^p$ as the total absolute change in the target value in a line drawn between data points:
\begin{equation}\label{eq:dpath}
d_{path}({\bf x},{\bf x}')=
\int\nolimits_0^1{\left|\frac{\partial f(t{\bf x}'+(1-t){\bf x})}{\partial t}\right| dt},
\end{equation}
where the function $f$ is defined in Eq. \eqref{eq:f}. 
We define the {\em embedding distance} between data points as Manhattan distance in the embedding space,
\begin{equation}\label{eq:demb}
 d_{emb}({\bf x},{\bf x}')=\sum\nolimits_{j=1}^m{\left|f_j({\bf x})-f_j({\bf x}')\right|}.
\end{equation}
The embedding distance can be considered an upper bound for a path distance.
The distances satisfy
\begin{equation}\label{eq:bounds}
\left|f({\bf x}')-f({\bf x})\right|\leq
d_{path}({\bf x},{\bf x}')\leq
d_{emb}({\bf x},{\bf x}').
\end{equation}
The lower bound of Eq. \eqref{eq:bounds} is tight when $f(t{\bf x}'+(1-t){\bf x})$ is non-decreasing or non-increasing function in $t\in[0,1]$. The upper bound is tight when, additionally, $f_j(t{\bf x}'+(1-t){\bf x})$ are for all $j\in[m]$ non-decreasing (or all non-increasing) functions of $t$; see Lemma \ref{lem:bounds} for a proof and Sect. \ref{sec:vis2} for visualizing the embedding space.

In the next section, we will experimentally demonstrate the usefulness of the above properties and provide some theoretical footing for them.

\subsection{Computational Complexity}\label{sec:complexity}

Solving the optimization problem of Eq. \eqref{eq:argmin} takes $O(npk)$ time, where $k$ is the number of optimization iterations the LBFGS algorithm takes (we take $k$ to be constant since, in practice, we restrict it to a reasonable value). The time complexity of the optimization process described in Sect. \ref{sec:learn} is therefore $O(npm)$.

\subsection{If There Is no Non-Linearity}
\label{sec:nonlin}

If the model is linear, i.e., $g(z)=z$, the optimization problem of Eq. \eqref{eq:argmin} for $f_1({\bf x})$ reduces to OLS linear regression for regression with a quadratic loss function and to the standard logistic regression for classification with the logistic loss function; see Sect. \ref{sec:linear} for a proof.

\subsection{Relation to Other Boosting Algorithms}
\label{sec:relboost}

Almost all boosting algorithms have a model in the form of Eq. \eqref{eq:f}, consisting of $m$ weak learners with high bias and low variance, which are trained sequentially (as in Sect. \ref{sec:learn}) to model the errors of subsequent modeling iterations. Typical choices for the loss function are, in addition to those mentioned in Sect. \ref{sec:def}, exponential loss $l_{exp}(y,y')=e^{-yy'}$ used by AdaBoost \cite{freund1995desiciontheoretic}, and hinge loss $l_{hinge}(y,y')=\max{(0,1-yy')}$ which can be thought of as non-smooth variant of the logistic loss. In many state-of-the-art algorithms, such as {\sc xgboost}, the weak learner is often a tree-based classifier and not a perceptron-like entity of Eq. \eqref{eq:fj}.

In principle, we could define embedding and the embedding distance similarly for any such boosting algorithm. Instead of tree-based approaches, {\sc gbmap} uses weak learners based on a simple perceptron structure to generate the embedding coordinates, which provide new features via simple transformation of the original data, automatically ignoring non-relevant directions. These coordinates have the following desirable properties over the tree-based weak learners: the perceptrons provide smooth (not piece-wise constant) transformation of the covariate space; thus, neighboring points are not likely to overlap in the embedding. Due to the construction, the predicted target value increases or decreases linearly when we move outside the training data, with the slope being proportional to the importance of the direction to the supervised learning task, helping us to detect data points with potentially large prediction errors.

\section{Numerical Experiments}

\subsection{Datasets and Algorithms}
\label{sec:dataalgo}

\begin{table}[h]
\caption{Datasets used in the experiments.}
\label{tab:reg-data}
\begin{center}
\begin{small}
\begin{sc}
\begin{tabular}{lrccc}
\toprule
Dataset & $n$ &  $p$ & Target \\
\midrule
abalone & 4\ 177 & 8 & $\mathbb{R}$ \\
air quality &7\ 355 & 11 & $\mathbb{R}$ \\
autompg & 392 & 8 & $\mathbb{R}$ \\
california & 20\ 640 & 7 & $\mathbb{R}$ \\
concrete & 1\ 030 & 8 & $\mathbb{R}$ \\
cpu-small & 8\ 192 & 12& $\mathbb{R}$ \\
qm9-10k & 10\ 000 & 27 & $\mathbb{R}$ \\
superconductor & 21\ 263 & 81 & $\mathbb{R}$ \\
synth-cos-r & 200\ 000 & 200 & $\mathbb{R}$ \\
wine-red & 1\ 599 & 11 & $\mathbb{R}$ \\
wine-white & 4\ 898 & 11 & $\mathbb{R}$ \\
breast-cancer & 569 & 30 & $ \pm 1$\\
diabetes & 768 & 8 & $\pm 1$\\
eeg-eye-state & 14\ 980 & 14 & $\pm 1$\\
german-credit & 1\ 000 & 20 & $\pm 1$\\
higgs-10k & 10\ 000 & 28 & $\pm 1$\\
synth-cos-c & 200\ 000 & 200 & $\pm 1$\\
\bottomrule
\end{tabular}
\end{sc}
\end{small}
\end{center}
\vskip -0.1in
\end{table}

We obtained ten real-world regression and five classification datasets shown in Tab. \ref{tab:reg-data} from the {\sc UCI} repository \cite{kellyuci}, {\sc OpenML} \cite{OpenML2013}, and {\sc scikit-learn} \cite{scikit-learn}, described in more detail in Sect. \ref{sec:data}. 

We additionally used a synthetic data ${\textsc{synth-cos-r}}(n,p)$ and ${\textsc{synth-cos-c}}(n,p)$ obtained as follows. We generated {\sc synth-cos-r} by first sampling the data matrix ${\bf X}\in{\mathbb{R}}^{n\times p}$ from a Gaussian distribution with zero mean and unit variance and let ${\bf y}=\alpha \cos{({\bf X})}{\bf u}$, where $\cos{({\bf X})}\in{\mathbb{R}}^{n\times p}$ denotes element-wise $\cos$, $\alpha \in \mathbb{R}$, and ${\bf u}\in{\mathbb{R}}^p$ is a random unit vector. Finally, we center ${\bf y}$ so that it has zero means. We used $\alpha =5$ throughout this work. The classification dataset {\sc synth-cos-c} is obtained by passing the regression target values to the logistic function to obtain class probabilities (see Sect. \ref{sec:linear}) and then randomly sampling the classification target values $\{-1,+1\}$ based on these probabilities. The ``default'' values are $n=2\cdot 10^5$ and $p=200$, and we have also included the intercept term in ${\bf X}$.

The {\sc drift} datasets for Sect. \ref{sec:ood} were obtained as follows. Given a real-world dataset, we split it into training and test datasets in a way that induces drift. We first find the most important feature to split the dataset that potentially produces drift. Specifically, for each feature, we split the data set into \texttt{a} and \texttt{b} of equal size according to the increasing order of that feature. We then further randomly split  \texttt{a} into  \texttt{a1},  \texttt{a2} where   \texttt{a1} and  \texttt{a2}
are of roughly the same size. We then drop the chosen feature from all mentioned subsets and train a {\sc gbmap} regressor (or classifier) on  \texttt{a1} and compute the loss (squared loss if regression, logistic loss if classification) of its predictions on  \texttt{a2} and  \texttt{b}. By construction, \texttt{a2} is assumed to come from the same distribution as the training data, while  \texttt{b} potentially has some drift in its distribution. We then use the difference between the loss on \texttt{b} and the loss on  \texttt{a2} to measure how much drift the split induces and pick the feature that maximizes the drift.

We preprocessed all real-world datasets by subtracting the mean, dividing by standard deviation, one-hot encoding the categorical covariates, and adding the intercept terms (for {\sc gbmap}).
We used OLS and logistic regression for the baseline comparisons in regression and classification tasks. We selected {\sc xgboost} as the main comparison for both tasks. We chose the OLS and logistic regression implementations from {\sc scikit-learn} \cite{scikit-learn} and used {\sc xgboost} from {\sc xgboost} Python library \cite{chen2016xgboost}.

We selected model hyperparameters by random search with 5-fold cross-validation (for OLS, no tuning is required). We gave the random search a budget of 100 iterations for each algorithm. For hyperparameter tuning, {\sc synth-cos-r} and {\sc synth-cos-c} were downsampled to $n=6\cdot 10^4$, and random search was given 10 iterations for efficiency. For logistic regression, we tuned only the $L_2$-penalty term. For {\sc xgboost}, we selected the number of boosting iterations, max tree depth and subsampling rate for hyperparameter tuning. Finally, for {\sc gbmap}, we selected $m$, $\beta$, $\lambda$, and the maximum iterations for LBFGS. The ranges or distributions for parameter values are specified in Tab. \ref{tab:hp-grid}.

\begin{table}[h]
\caption{Hyperparameter ranges for {\sc lr} (logistic regression, {\sc xgboost} and {\sc gbmap}. $U(\cdot)$ denotes an uniform distribution.}
\label{tab:hp-grid}
\begin{center}
\begin{sc}
\begin{tabular}{lccc}
\toprule
model & parameter & range \\
\midrule
lr & $L_2$-penalty & $U(10^{-5},1$) \\
xgboost & n-boosts & 100--2000 \\
xgboost & maxdepth & 1--5 \\
xgboost & subsample & $[0.5, 0.6, ..., 1]$ \\
gbmap & $m$ (n-boosts) & 2--150 \\
gbmap & softplus $\beta$ & 1--20 \\
gbmap & maxiter & [200, 400] \\
gbmap & $\lambda$ & $U(0,10^{-2})$ \\
\end{tabular}
\end{sc}
\end{center}
\end{table}

 We also compared the {\sc gbmap} distance to Euclidean distance in the supervised learning setting. To this end, we applied $k$-Nearest Neighbors ($k$-NN) with Euclidean distance and {\sc gbmap} $L_1$ distance ($k=10$ for both variants). We chose the {\sc gbmap} hyperparameters similarly as above by random search but used only 50 iterations, and $m$ was sampled from $[2,50]$. We downsampled the real-world datasets to $n=10^4$ and used ${\textsc{synth-cos-r}}(10^4, 20)$ and ${\textsc{synth-cos-c}}(10^4, 20)$ for computational reasons.

\subsection{Scaling}\label{sec:scaling}

We conducted an experimental evaluation of the runtime performance (Tab. \ref{tab:scaling}) of {\sc gbmap}, comparing it with {\sc lol}, {\sc ivis} and {\sc pca}. The evaluation was performed on a computing cluster, on which each run was allocated two processors and 64GB of RAM. We generated synthetic using ${\textsc{synthetic-cos-c}}(n,p)$, with varying parameters $n$ and $p$. The algorithm runtimes were averaged over ten repeated runs. We set $m=2$ in {\sc gbmap}, and the number of embedding components was also set to $2$ for all other methods. We set $\lambda=10^{-3}$, $\beta=1$ and {\sc maxiter} $=100$ for {\sc gbmap}.  For {\sc ivis}, we used the default number of epochs $1000$ from the official implementation with an early stopping set to $5$ epochs. {\sc gbmap} scales roughly as $O(np)$ and is quite fast. On data dimensionality scaling {\sc gbmap} is roughly comparable to {\sc lol}. Meanwhile, {\sc ivis} scales poorly, taking an order of magnitude longer to run.

\begin{table}[h]
\caption{Wall clock runtimes in seconds, averaged over ten runs. Lower values are better.}
\label{tab:scaling}
\begin{center}
\begin{small}
\begin{sc}
\begin{tabular}{rrrrrr}
\toprule
$n$ & $p$ & {\sc gbmap} & {\sc ivis} & {\sc lol} & {\sc pca} \\
\midrule
\( 10^5 \)& 100 & 5.8 & 276.8 & 1.7 & 0.9 \\
\( 10^5 \)& 200 & 4.5 & 270.1 & 3.1 & 1.5 \\
\( 10^5 \) & 400 & 7.8 & 331.0 & 6.5 & 3.1 \\
\( 10^5 \) & 800 & 22.5 & 499.5 & 17.6 & 8.5 \\
\( 10^5 \) & 1600 & 79.1 & 932.5 & 53.9 & 27.3 \\
\( 10^5 \) & 3200 & 204.2 & 1688.2 & 195.8 & 106.3 \\

\(10^5\)         & 25 & 1.4 & 286.4 & 0.2 & 0.1 \\
\(5 \cdot 10^5\) & 25 & 14.3 & 1719.3 & 1.1 & 0.5 \\
\(10^6 \)        & 25 & 14.2 & 4710.6 & 2.3 & 0.9 \\
\(5 \cdot 10^6\) & 25 & 174.8 & 18223.8 & 11.2 & 4.9 \\
\(10^7\)         & 25 & 204.9 & 26623.5 & 22.3 & 9.8 \\
\bottomrule
\end{tabular}
\end{sc}
\end{small}
\end{center}
\vskip -0.1in
\end{table}

\subsection{Regression and Classification}\label{sec:regression}

The regression and classification experiments' results ($R^2$ and accuracy) are in the first block of columns of Tab. \ref{tab:reg}. We see that  {\sc gbmap} is roughly comparable to {\sc xgboost} across all datasets when we use {\sc gbmap} directly as a regressor/classifier. If we use the embedding distance produced by {\sc gbmap} in Eq. \eqref{eq:demb} as the proposal distance for $k$-NN, it can improve the performance of $k$-NN (the second block of columns of Tab. \ref{tab:reg}). For ``easier'' datasets where almost all features are relevant, $k$-NN gives comparable results with the Euclidean and embedding distances of Eq. \eqref{eq:demb}, however, for more realistic datasets where this is not the case (e.g., {\sc california}, {\sc concrete}, {\sc qm9-10k}, {\sc synth-cos-r}, {\sc synt-cos-c}) the embedding distance outperforms the Euclidean one. Indeed, if we added random features to any datasets, the advantage of $k$-NN with the embedding distance would become even more prominent.

\begin{table*}[h]
\caption{Regression results ($R^2$, top) and classification results (accuracy, bottom). Columns {\sc gbmap}, {\sc linreg} (OLS or logistic regression), and {\sc xgboost} show the performance of the respective regression/classification algorithms. The column {\sc knn} shows the $k$-NN performance with Euclidean distance and column {\sc knn-gbmap} with the embedding distance of Eq. \eqref{eq:demb}. 
AUC of {\sc gbmap} and {\sc euclid} drifters are shown on the right-hand columns. An AUC of $0.5$ corresponds to random guessing.
Higher values are better for all columns.}
\label{tab:reg}
\begin{center}
\begin{small}
\begin{sc}
\begin{tabular}{lccc|cc|cc}
\toprule
 dataset & gbmap & linreg & xgboost & knn & knn-gbmap  & auc-gbmap & auc-euclid \\
\midrule
abalone & \textbf{0.577} & 0.526 & 0.535 & 0.517 & {\bf 0.556} & 0.95 & \textbf{0.96} \\
airquality & \textbf{0.922} & 0.905 & 0.912 & {\bf 0.914} & 0.912 & \textbf{0.98} & 0.96 \\
autompg & \textbf{0.863} & 0.784 & 0.807 & 0.840 & {\bf 0.855}& \textbf{0.94} & 0.91 \\
california & 0.755 & 0.602 & \textbf{0.841} & 0.689 & {\bf 0.767} & \textbf{0.74} & 0.64 \\
concrete & 0.907 & 0.582 & \textbf{0.917} & 0.634 & {\bf 0.871} & 0.73 & \textbf{0.81} \\
cpu-small & 0.974 & 0.713 & \textbf{0.978} & 0.945 & {\bf 0.974} & \textbf{0.91} & 0.84 \\
qm9-10k & 0.701 & 0.471 & \textbf{0.707} & 0.574 & {\bf 0.689} & \textbf{0.96} & 0.87 \\
superconductor & 0.874 & 0.730 & \textbf{0.911} & 0.852 & {\bf 0.869} & \textbf{0.98} & 0.97 \\
synth-cos-r & \textbf{0.452} & 0.00 & 0.025 & 0.411 & {\bf 0.748} & $-$&$-$\\
wine-red & \textbf{0.396} & 0.358 & 0.395 & 0.292 & {\bf 0.362} & \textbf{0.85} & 0.77 \\
wine-white & 0.395 & 0.290 & \textbf{0.432} & 0.335 & {\bf 0.363} & \textbf{0.95} & 0.89 \\
\midrule
breast-cancer & 0.965 & 0.953 & \textbf{0.971} & 0.947 & {\bf 0.982} & 0.51 & \textbf{0.59} \\
diabetes & 0.766 & \textbf{0.775} & 0.758 & 0.727 & {\bf 0.784} & 0.63 & \textbf{0.75} \\
eeg-eye-state & 0.755 & 0.642 & \textbf{0.944} & 0.773 & {\bf 0.801} & 0.66 & \textbf{0.74} \\
german-credit & 0.730 & \textbf{0.767} & 0.733 & 0.703 & {\bf 0.713} & \textbf{0.65} & 0.63 \\
higgs-10k & 0.681 & 0.639 & \textbf{0.703} & 0.615 & {\bf 0.661} & \textbf{0.70} & 0.67 \\
synth-cos-c & \textbf{0.583} & 0.50 & 0.582 & 0.593 & {\bf 0.698} & $-$ & $-$\\
\bottomrule
\end{tabular}
\end{sc}
\end{small}
\end{center}
\vskip -0.1in
\end{table*}

\subsection{Supervised Learning Features}\label{sec:features}

We evaluated {\sc gbmap} embeddings similarly to \cite{vogelstein2021Supervised} by measuring the generalization error of a supervised model trained on embeddings. As the supervised learning models, we selected OLS linear regression for regression tasks and logistic regression for classification to demonstrate that even simple linear models can perform competitively with good features. We used half of the data for training and the other half to estimate the generalization error. The generalization error ($R^2$ or accuracy) was calculated over five repeated splits. We compared the {\sc gbmap} embeddings in a supervised learning setting to other supervised embedding methods such as {\sc lol}, {\sc ivis}, and unsupervised {\sc pca}. We applied {\sc lol} only for classification datasets, as the method assumes class-label target values.

For {\sc ivis}, we set the maximum epochs to 1000 with early stopping set to $5$ epochs, and we used mean absolute error as the supervision metric. We selected {\sc gbmap} $\beta$ using random search with 10 iterations and set $\lambda=10^{-3}$ and maximum iterations for LBFGS optimizer $\text{\sc maxiter}=400$. For the experiment, we selected datasets where the linear model compared poorly to more complex models (Sect. \ref{sec:regression}). We split the data randomly into half; one half was used to train the model and find the embedding, and the other was used to estimate generalization error. The above was repeated five times to control the randomness.

Fig. \ref{fig:reg_features_reg} presents the improved performance of linear regression and logistic regression when using new features extracted by various methods. The complete results for all datasets are in Sect.  \ref{sec:appendix_features}. We can see that {\sc gbmap} provides good features for supervised learning, comparable to {\sc ivis}. On the other hand, {\sc lol} cannot improve the logistic regression at all. Likewise, as expected, {\sc PCA} is merely a dimensionality reduction method; hence, its performance is worse than the baselines that use all original features.

Using {\sc gbmap} features, simple linear models can narrow the gap between complex and black-box models. Note that the {\sc gbmap} transformation is not restricted by the number of covariates $p$, unlike {\sc pca} and {\sc lol}.

\begin{figure}[h]
    \centering
    \subfloat[{\sc qm9-10k}\label{fig:reg_features_reg:a}]{
     
	\includegraphics[width=0.45\textwidth]{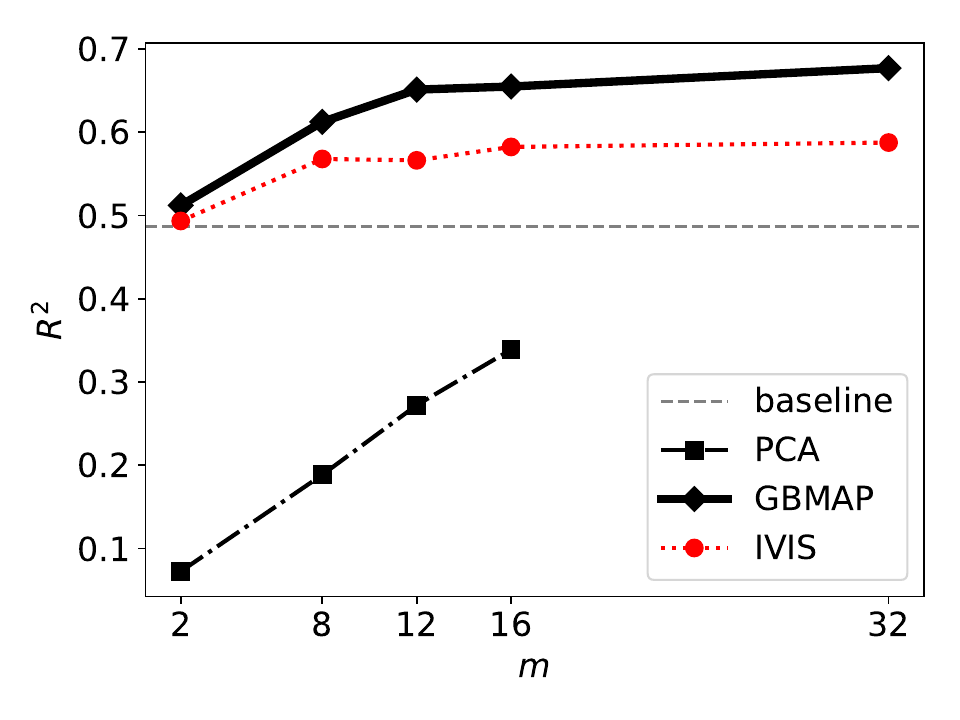}
	}
 \subfloat[{\sc higgs-10k}\label{fig:reg_features_reg:b}]{
	\includegraphics[width=0.45\textwidth]{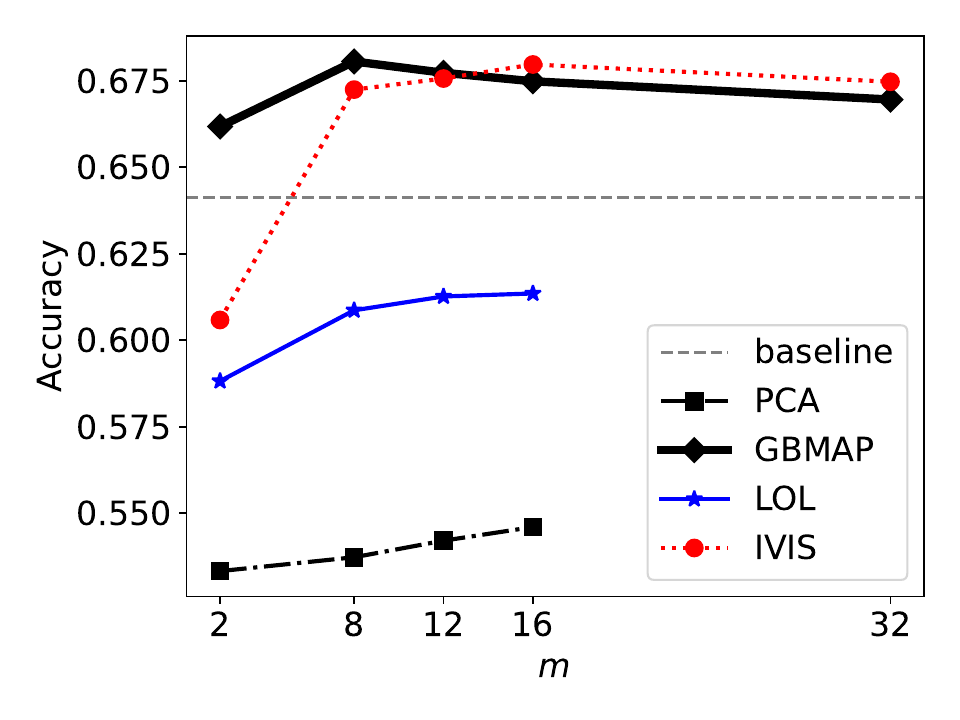} 
	}
    \caption{Embeddings as features for OLS regression ({\sc qm9-10k}, $R^2$, (a)) and logistic regression ({\sc higgs-10k}, accuracy, (b)). The {baselines} are OLS regression (a) and logistic regression (b) trained on the original data. The {\sc gbmap} transformation is not restricted by the data $p$ (the number of covariates), unlike {\sc pca} and {\sc lol} and can be used to transform the data to arbitrary dimensions. The {\sc higgs-10k} and {\sc qm9-10k} has $p<32$ hence, the lines for {\sc pca} and {\sc lol} end at $m=16$.}
\label{fig:reg_features_reg}
\end{figure}

\subsection{Out-of-Distribution Detection}
\label{sec:ood}

\begin{figure*}[h]
     \centering
     \subfloat[{\sc gbmap} drifter\label{fig:gbmapdrifter:a}]{
	\includegraphics[width=0.48\textwidth]{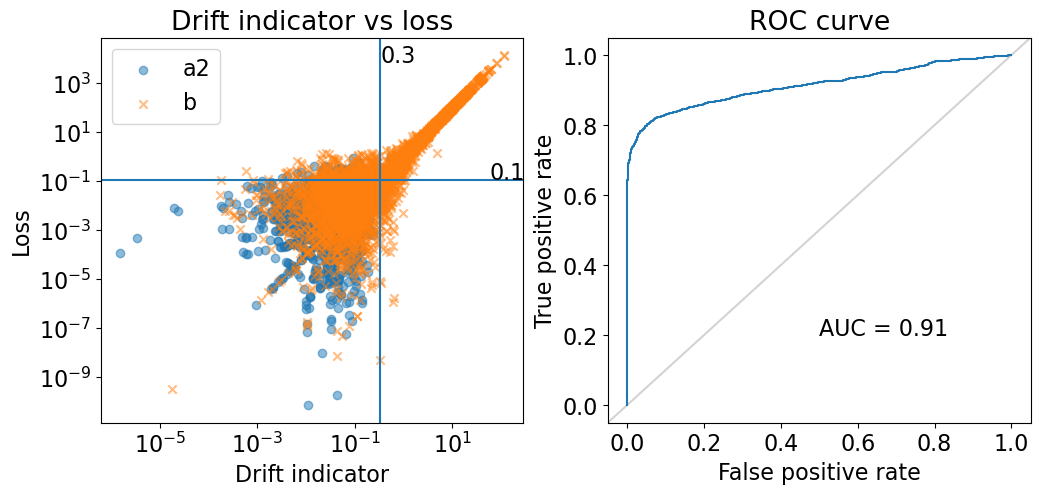} 
	}
 \subfloat[{\sc euclid} drifter\label{fig:gbmapdrifter:b}]{
	\includegraphics[width=0.48\textwidth]{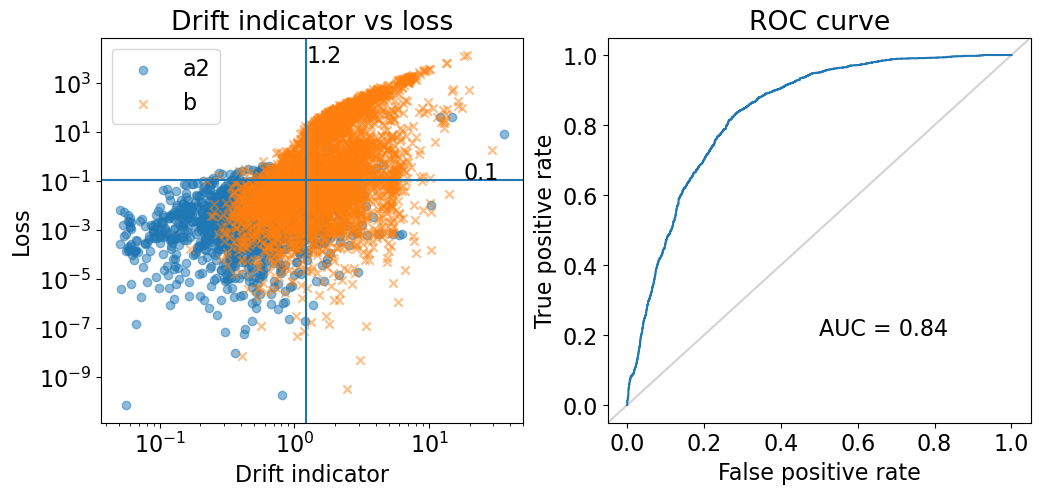} 
	}
	\caption{The drift indicator against the loss figures (left) and ROC curves (right) for the {\sc gbmap} (left column) and {\sc euclid} (right column) drifters on {\sc cpu-small} dataset. The horizontal line denotes our chosen concept drift threshold, while the vertical line indicates the drift indicator threshold that leads to maximal $F_1$ score (not used in the analysis), as in \cite{oikarinenDetectingVirtualConcept2021}. The blue spheres are the data from the in-distribution set \texttt{a2}, and the orange crosses are the data from the out-of-distribution set \texttt{b}. The {\sc gbmap} drifter detects drift with a high AUC of $0.91$, while the {\sc euclid} has AUC of $0.84$}
	\label{fig:gbmapdrifter}
\end{figure*}

In this subsection, we show how the embedding distance can be used to detect concept drift, i.e., when the test data follows a different distribution to the training data, resulting in deteriorating prediction performance. Following \cite{oikarinenDetectingVirtualConcept2021}, we frame the drift detection problem as a binary classification problem; we say there is concept drift if a regression or classification loss on a new data point exceeds a pre-defined threshold. We introduce concept drift in real-world datasets using the splitting scheme described in Sect. \ref{sec:dataalgo}, where we denote the training set by \texttt{a1}, the in-distribution test set \texttt{a2}, and the out-of-distribution set \texttt{b}. We then train {\sc gbmap} on \texttt{a1} and want to detect drift on the test datasets \texttt{a2} and \texttt{b}. We define a point as being out-of-distribution if the {\sc gbmap} error (squared loss if regression, logistic loss if classification) at that point is larger than a threshold here we use $95 \%$ quantile of the losses on \texttt{a2}. For real-world applications, the threshold would be set by a domain expert.

Our idea to detect drift is as follows. {\sc gbmap} and a $k$-NN model should give similar outputs within the training distribution. However, when moving outside the distribution, the {\sc gbmap} prediction increases or decreases linearly, with the slope depending on the importance of the direction to the supervised learning task. On the other hand, the $k$-NN prediction remains roughly constant when we move far outside the training data distribution. Therefore, we propose using the difference between the {\sc gbmap} and $k$-NN predictions as a drift indicator: if the prediction depends strongly on the prior modeling assumptions, we risk a larger-than-expected loss. 

As the ground-truth loss, in regression, we use the squared difference between the actual target and the {\sc gbmap} prediction $(y-f(x))^2$ and in classification, we use the squared difference between the actual score $s$ and the {\sc gbmap} prediction. Here, by $s$, we refer to the values passed to the sigmoid to get class probabilities. In real-world datasets, as the actual score is unavailable, we use the score function estimated via logistic regression, i.e., $\hat{s}_i = {\bf w}^T {\bf x}_i$, where we obtain the parameters ${\bf w}$ by optimization over complete data ${\bf w} = \text{arg min}_{\bf w} \ n^{-1} \sum_{i=1}^n \log(1+e^{-y_i \hat{s}_i})$.

Given a test point ${\bf x}^*$, we denote by $\mathcal{N}({\bf x}^*)\subseteq[n]$, where $\left|\mathcal{N}({\bf x}^*)\right|=k$, the $k$ nearest training data points to ${\bf x}^*$ using the distance produced by {\sc gbmap} embedding of Eq. \eqref{eq:demb}. For regression, we use $f_{kNN}({\bf x}^*)=\sum_{i \in \mathcal{N}({\bf x}^*) }{y_i}/k$. For classification, the $k$-NN prediction of the score at ${\bf x}^*$ is given by $f_{kNN}({\bf x}^*)=\sum\nolimits_{i \in \mathcal{N}({\bf x}^*) }{f({\bf x}_i)}/k$. The {\em drift indicator} for the {\em {\sc gbmap} drifter} is the difference between the {\sc gbmap} prediction at ${\bf x}^*$ and this $k$-NN prediction, i.e., $\left|f({\bf x}^*)-f_{kNN}({\bf x}^*)\right|$. By construction, the features irrelevant to the supervision tasks will likely be ignored by {\sc gbmap} and thus have little effect on the outcome. 

As a baseline comparison, we use the Euclidean distance in the original data space from the test point ${\bf x}^*$ to the $k$:th nearest training data point ({\em {\sc euclid} drifter}). We chose $k=5$ for both drifters. We use the following parameters for {\sc gbmap} across all datasets: number of boosting steps $m=20$, softplus $\beta=5$, Ridge regularization $\lambda=10^{-3}$, maximum number of the LBFGS optimizer iterations $\text{\sc maxiter}=200$. We downsample the {\sc superconductor} dataset for this drift experiment to $n=10^4$.

Fig. \ref{fig:gbmapdrifter} depicts the scatter plots between the error and the drift indicators as well as the Receiver Operating Characteristic (ROC) curves \cite{fawcett2006introduction} of {\sc gbmap} and {\sc euclid} drifters for the {\sc cpu-small} dataset. The {\sc gbmap} drift indicator is highly correlated with the error, as desired, superior to the Euclidean indicator. Figures for the other datasets are in Sect. \ref{sec:drift}. The Area Under the Curve (AUC) values of the {\sc gbmap} and {\sc euclid} drifters when detecting concept drift for regression and classification, respectively, are shown in the right-hand columns of Tab. \ref{tab:reg}.
For the regression datasets, the {\sc gbmap} drifter is superior to the {\sc euclid} drifter, while for classification datasets, both drifters perform similarly.

We note that for some datasets, the splitting scheme described in \ref{sec:dataalgo} fails to introduce any concept drift. This can happen when there are no clearly important covariates to the given supervised task, e.g., for {\sc california} only around $12 \%$ of the data points in the out-of-distribution set \texttt{b} have loss larger than the $95 \%$ quantile threshold, compared to the in-distribution set \texttt{a2}, which has by construction $5 \%$ of points labeled as drift. Also, the datasets {\sc diabetes}, {\sc german-credit} contain only slight drift and {\sc breast-cancer} contains no drift at all (the mean loss for \texttt{b} is lower than \texttt{a2}).

\section{Other Uses for {\sc gbmap}}\label{sec:other}
\subsection{Supervision via Differences Between Models}\label{sec:f0}

The ``default'' choice is to have $f_0({\bf x})=0$, but if we want for example, to study the errors made by a pre-trained supervised learning model $h:{\mathbb R}^p\to{\mathbb R}$, we can set $f_0({\bf x})= h({\bf x})$, in which case the embedding tries to model the residuals of the model, effectively capturing the model's shortcomings.

As a real-world example, consider a simple linear regression model $h_1$ with a moderate $R^2$ score and a black-box model (e.g., a deep neural network) $h_2$ with a high $R^2$ score. The challenge is to identify (semi)explainable features that can enhance the performance of the simple linear model, bringing it closer to the black-box model. In this case, we can set $f_0=h_1$ and treat the black-box model as the target function, i.e., $y_i=h_2({\bf x}_i)$. It is worth noting that the objective here is to interpret and find new features capturing the behavior of the black-box model, not the original data. The black-box model can provide labeled data for any region of interest, allowing us to create a (semi)white-box model that approximates the black-box model's behavior in various regions, not just the training region. In another example, we can use {\sc gbmap} to find regions where these models exhibit the most significant differences. Of course, the absolute difference $\vert h_1({\bf x}) - h_2({\bf x}) \vert$ can be computed point-wise, but it does not give regions in a principled manner. We again turn to {\sc gbmap}, boosting from one model to another. By construction, {\sc gbmap} gives a sequence of hyperplanes $\{{\bf x}\in{\mathbb{R}}^p\mid{\bf w}_j^\intercal {\bf x} = 0 \}$ that separate the feature space. Notably, the half-space activated by the softplus function indicates regions where the two models diverge, particularly in the initial {\sc gbmap} iterations. If a region is never activated, it suggests that the two models are likely similar in that specific region.

\subsection{Explainability}
\label{sec:explainability}

\begin{figure}[h]
    \centering
    \includegraphics[width=0.55\textwidth]{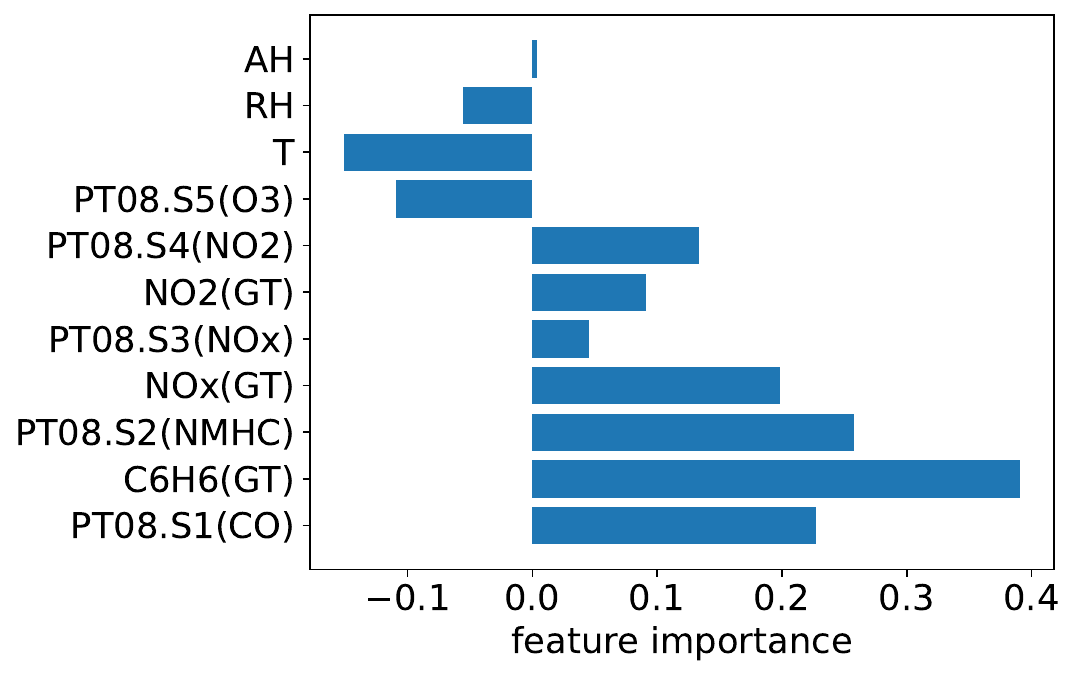}
    \caption{Local feature importance in terms of local linear regression coefficients in the decision of {\sc gbmap} for the 50th training data point in the {\sc airquality} dataset. We have omitted the intercept term from the plot.}
    \label{fig:xAI}
\end{figure}

Unlike black-box models that require explainable AI ({XAI}) tools to give insight into why these models predict specific outcomes  \cite{guidotti:2018:a,bjorklund2023slisemapdemo}, our method is interpretable with its built-in {XAI}.  
The softplus-nonlinearity is a smooth approximation of the ReLU, which separates the feature space into two parts by a hyperplane and inactivates one part while letting another part be as-is. The feature space is therefore partitioned into several regions by each weak learner, and in each area, the {\sc gbmap} prediction is roughly piece-wise linear. Consequently, we can obtain the importance of each feature in that region and have local interpretable models for each decision in terms of the linear or logistic regression coefficients. For example, Fig. \ref{fig:xAI} shows the regression coefficients of each feature in the {\sc gbmap} prediction at the $50$th training data point of the {\sc airquality} dataset. We can, from this figure, infer that \texttt{AH} plays no role in the prediction of {\sc gbmap} at this point, while \texttt{C6H6(GT)} is the most essential feature for this prediction. Recall that all features have been normalized to unit variance; hence, the regression coefficients can be interpreted directly as feature importance's.

\section{Conclusion}\label{sec:conclusion}

We have proposed a novel supervised dimensional reduction method based on a boosting framework with simple perceptron-like weak learners. As a side-product, we obtain a competitive and interpretable supervised learning algorithm for regression and classification. The transformations defined by the boosting structure of weak learners can be used for feature creation or supervised dimensional reduction.

For finding features for regression or classification, we are at least as good as the competition, such as {\sc ivis}. Our advantage is the speed and more interpretable features obtained by a simple perception transformation. Indeed, if we drop the interpretability requirement, it would be trivial to make a good feature just by including the target variable predicted by a powerful black box regression/classification algorithm as one of the features! 
The {\sc gbmap} transformation induces a distance that ignores irrelevant directions in the data, which can be used to improve distance-based learning algorithms such as $k$-Nearest Neighbors. We also showed that {\sc gbmap} can reliably detect data points with potentially large prediction errors, which is important in practical applications, including concept drift.

Interesting future directions include a more in-depth study of {\sc gbmap} to detect concept drift and quantify the uncertainty of regressor/classifier predictions. Another interesting avenue for future work is visualization: how the found embedding and the embedding distance could also be used to make supervised embeddings of the data for visual inspection (initial example shown in Sect. \ref{sec:vis2} and Fig. \ref{fig:vis}).

\begin{credits}
\subsubsection{\ackname} We acknowledge the funding by the University of Helsinki and the Research Council of Finland (decisions 346376 and 345704).

\subsubsection{\discintname}
The authors have no competing interests to declare 
relevant to the content of this article
\end{credits}

\bibliographystyle{splncs04}
\bibliography{bib}

\newpage
\appendix

{\bf Appendix.}
The main text is self-contained and can be read and understood without the following appendix. Here, we present proofs and derivations for completeness. We also have included some complementary experimental results that we decided to exclude from the main text for compactness and lack of space.

\section{{\sc gbmap} Reduces to OLS Linear and Logistic Regression if There Is no Non-Linearity}\label{sec:linear}

In this section, we study the case where $f_0({\bf x})=0$ and there is no non-linearity, i.e., $g(z)=z$. We show that in that case, {\sc gbmap} reduces to OLS linear regression and logistic regression, both with Ridge regularization, and that after the first boosting iteration ($j\ge 2$), the coefficients are of the order of $O(\lambda)$. Nonlinearity is, therefore, essential to have non-trivial solutions with more than one embedding coordinate. We state this formally with the following Lemma.
\begin{lemma}\label{lem:linear}
    The optimization problem for $f_1$ ($j=1$) of Eq. \eqref{eq:argmin} reduces for ordinary least squares linear regression with $b_1=1$ for the quadratic loss and standard logistic regression for the logistic loss, both with Ridge regularization, if there is no nonlinearity, i.e., $g(z)=z$. The parameters of subsequent weak learners $f_j$  for $j\ge 2$ are proportional to $a_j=O(\lambda^{j-1})$ and ${\bf w}_j=(O(\lambda^{j-1}),\ldots,O(\lambda^{j-1}))^\intercal$ and they vanish in the absence of the regularization term (i.e., if $\lambda=0$).
\end{lemma}
\begin{proof}

\noindent {\bf Regression.}
For the regression problem and quadratic loss $l_{quadratic}(y,y')=(y-y')^2$, the optimization problem of Eq. \eqref{eq:argmin} reduces to
\begin{equation}\label{eq:argmin_ols}
a_1,b_1,{\bf w}_1
=\argmin\nolimits_{a_1,b_1,{\bf w}_1}{\left(
n^{-1}\sum\nolimits_{i=1}^n{\left(y_i-a_1-b_1{\bf w}_1^\intercal{\bf x}_i\right)^2}+
\lambda\lVert{\bf w}_1\rVert_2^2\right)}\\
\end{equation}
We can, without loss of generality, take $b_1=1$, because changing  ${\bf w}_1\leftarrow b_1{\bf w}_1$ and $b_1\leftarrow 1$ leaves the loss unchanged.

Eq. \eqref{eq:argmin_ols} defines OLS linear regression with Ridge regularization. In the absence of regularization ($\lambda=0$), OLS linear regression gives the unbiased estimator, for which reason $a_j=0$ and ${\bf w}_j=(0,\ldots,0)^\intercal$ for $j\ge 2$. However, in the presence of regularization, we have $a_j=O(\lambda)$ and ${\bf w}_j=(O(\lambda),\ldots,O(\lambda))^\intercal$ for $j\ge 2$ due to the fact that Ridge regression provides a biased estimator. The subsequent terms ($j>2$) are proportional to $O(\lambda^{j-1})$, respectively.

\noindent {\bf Classification.}
The ``standard'' logistic regression is a generalized linear model using the logit link function and maximum likelihood loss for the binomial distribution.

The probability that $y=+1$ given ${\bf x}$ is then given by $p(y=+1\mid{\bf x})=\sigma({\bf w}^\intercal{\bf x})$ and a probability of $y=-1$ by $p(y=-1\mid{\bf x})=1-\sigma({\bf w}^\intercal{\bf x})$, where $\sigma(z)=1/(1+e^{-z})$ is the sigmoid function or the inverse of the logit link function. The log-loss for observation $({\bf x},y)$ is then given by
\begin{equation}
\begin{split}
L&
=-I(y=+1)\log{p(y=+1\mid{\bf x})} -I(y=-1)\log{p(y=-1\mid{\bf x})}\\
&=-I(y=+1)\log{\left(\sigma({\bf w}^\intercal{\bf x})\right)}-I(y=-1)\log{\left(1-\sigma({\bf w}^\intercal{\bf x})\right)}\\
&=-I(y=+1)\log{\left(\sigma({\bf w}^\intercal{\bf x})\right)}-I(y=-1)\log{\left(\sigma(-{\bf w}^\intercal{\bf x})\right)}\\
&=-I(y=+1)\log{\left(\sigma(y{\bf w}^\intercal{\bf x})\right)}-I(y=-1)\log{\left(\sigma(y{\bf w}^\intercal{\bf x})\right)}\\
&=-\left(I(y=+1)+I(y=-1)\right)\log{\left(\sigma(y{\bf w}^\intercal{\bf x})\right)}\\
&=-\log{\left(\sigma(y{\bf w}^\intercal{\bf x})\right)}\\
&=\log{\left(1+e^{-y{\bf w}^\intercal{\bf x}}\right)}=l_{logistic}\left(y,{\bf w}^\intercal{\bf x}\right),
\end{split}
\end{equation}
where we have used $\sigma(z)=1-\sigma(-z)$ and $l_{logistic}(y,y')=\log{\left(1+e^{-yy'}\right)}$.

For the classification problem, the optimization problem of Eq. \eqref{eq:argmin} reduces to
\begin{equation}\label{eq:argmin_logistic}
a_1,b_1,{\bf w}_1
=\argmin\nolimits_{a_1,b_1,{\bf w}_1}{\left(
n^{-1}\sum\nolimits_{i=1}^n{l_{logistic}\left(y_i,a_1+b_1{\bf w}_1^\intercal{\bf x}_i\right)}+
\lambda\lVert{\bf w}_1\rVert_2^2\right)}.
\end{equation}
We can, without loss of generality, take $b_1=1$, because changing  ${\bf w}_1\leftarrow b_1{\bf w}_1$ and $b_1\leftarrow 1$ leaves the loss unchanged.

Eq. \eqref{eq:argmin_logistic} defines the standard logistic regression with logit link function and Binomial loss with Ridge regularization. In the absence of regularization ($\lambda=0$), the logistic regression gives the unbiased estimator, for which reason $a_j=0$ and ${\bf w}_j=(0,\ldots,0)^\intercal$ for $j\ge 2$. However, in the presence of regularization, we have $a_j=O(\lambda)$ and ${\bf w}_j=(O(\lambda),\ldots,O(\lambda))^\intercal$ for $j\ge 2$ because Ridge regression provides a biased estimator.
\end{proof}

As described in the proof above, in classification tasks with logistic loss, the predicted probability of $+1$ is given by $\hat p(y=+1\mid{\bf x})=\sigma(f({\bf x}))$, and the value of the response given the probability of one is given by as $f({\bf x})=\sigma^{-1}\left(\hat p(y=+1\mid{\bf x})\right)$, where the sigmoid function is given by $\sigma(z)=1/(1+e^{-z})$ and the logit function by $\sigma^{-1}(p)=\log{\left(p/(1-p)\right)}$. We can, therefore, use the logit function to transform outputs of a probabilistic classifier, outputting class probabilities, to the {\sc gbmap} response space $y$, if necessary, and the sigmoid function to obtain class probabilities from the response space.

\newpage

\section{Properties of the Distance Measures}\label{sec:distance_proof}

We define in Eq. \eqref{eq:demb} the {\em embedding distance} and in Eq. \eqref{eq:dpath} {\em path distance}. Here, we show the intuition behind the distances and prove the inequality of Eq. \eqref{eq:bounds}.

The ``default'' distance between points ${\bf x}\in{\mathbb{R}}^p$ and ${\bf x}'\in{\mathbb{R}}^p$ is given by the Euclidean distance $\lVert {\bf x}-{\bf x}'\rVert_2$ or the Manhattan distance $\lVert {\bf x}-{\bf x}'\rVert_1$. However, these distances have the undesirable property that they weigh all directions equally, including those irrelevant to the supervised learning task. For this reason, we have defined a path distance in Eq. \eqref{eq:dpath}. The path distance between ${\bf x}$ and ${\bf x}'$ is the total absolute change of the function $f$ when traversing a straight line $t{\bf x}'+(1-t){\bf x}$ from ${\bf x}$ to ${\bf x}'$, parameterized by $t\in[0,1]$. Next, we show some useful properties of the integral.
\begin{lemma}\label{lem:increasing}
If $f(t{\bf x}'+(1-t){\bf x})$ is non-decreasing or non-increasing function in $t\in[a,b]$ for some $a<b$ then 
\begin{equation}
    \int\nolimits_a^b{\left|\frac{\partial f(t{\bf x}'+(1-t){\bf x})}{\partial t} \right|dt}=\left|f({\bf x}_a)-f({\bf x}_b)\right|,
\end{equation}
where ${\bf x}_a=a{\bf x}'+(1-a){\bf x}$ and ${\bf x}_b=b{\bf x}'+(1-b){\bf x}$.
\end{lemma}
\begin{proof}
If $f(t{\bf x}'+(1-t){\bf x})$ is non-decreasing, its derivative is non-negative, we can drop $\left|\Box\right|$, and the integral is simply
\begin{equation}\label{eq:dpath_pos}
\begin{split}
\int\nolimits_a^b{\left|\frac{\partial f(t{\bf x}'+(1-t){\bf x})}{\partial t} \right|dt}&=
\int\nolimits_a^b{\frac{\partial f(t{\bf x}'+(1-t){\bf x})}{\partial t} dt}\\&=f({\bf x}_b)-f({\bf x}_a)\\&=\left|f({\bf x}_a)-f({\bf x}_b)\right|,
\end{split}
\end{equation}
where the second equality follows from the definition of the integral and the last from the fact that the function is non-decreasing, i.e., $f({\bf x}_b)\ge f({\bf x}_a)$.
We get the same expression if $f(t{\bf x}'+(1-t){\bf x})$ is non-increasing function in $t\in[a,b]$, i.e., the gradient is negative and $f({\bf x}_a)\ge f({\bf x}_b)$:
\begin{equation}\label{eq:dpath_neg}
\begin{split}
\int\nolimits_a^b{\left|\frac{\partial f(t{\bf x}'+(1-t){\bf x})}{\partial t} \right|dt}&=-
\int\nolimits_a^b{\frac{\partial f(t{\bf x}'+(1-t){\bf x})}{\partial t} dt}\\&=f({\bf x}_a)-f({\bf x}_b)\\&=\left|f({\bf x}_a)-f({\bf x}_b)\right|.
\end{split}
\end{equation}
\end{proof}

We have the following lemma for functions that are neither decreasing nor increasing.
\begin{lemma}\label{lem:dpath_sum}
Assume we can still split the axis $t\in[0,1]$ into $k$ segments, where $t_0=0\le t_1\le\ldots\le t_{k-1}\le t_k=1$ such that $f(t{\bf x}'+(1-t){\bf x})$ is either non-decreasing or non-increasing function in each of the $k$ intervals $t\in[t_{l-1},t_l]$ for all $l\in[k]$. The path distance can be expressed as a sum of absolute changes of the residual function:
\begin{equation}\label{eq:dpath_sum}
d_{path}({\bf x},{\bf x}')=
\sum\nolimits_{l=1}^k{\left|f(t_{l-1}{\bf x}'+(1-t_{l-1}){\bf x})-
f(t_l{\bf x}'+(1-t_l){\bf x})\right|}.
\end{equation}
\end{lemma}
\begin{proof}
The proof follows directly from Lemma \ref{lem:increasing}.
\end{proof}

The Lemma below proves Eq. \eqref{eq:bounds}.
\begin{lemma}\label{lem:bounds}
The path and embedding distances of Eqs. \eqref{eq:dpath} and \eqref{eq:demb} satisfy the bounds of Eq. \eqref{eq:bounds}, i.e.,$$
\left|f({\bf x}')-f({\bf x})\right|\leq
d_{path}({\bf x},{\bf x}')\leq
d_{emb}({\bf x},{\bf x}').
$$
\end{lemma}
\begin{proof}
The sum of Eq. \eqref{eq:dpath_sum} is lower-bounded by $\left|f({\bf x}')-f({\bf x})\right|$, because of the triangle inequality $|a-c|\le|a-b|+|b-c|$ for any $a,b,c\in{\mathbb{R}}$. According to Lemma \ref{lem:increasing}, the lower bound is tight when $f(t{\bf x}'+(1-t){\bf x})$ is increasing or decreasing function in $t\in[0,1]$.

We can rewrite the path distance as 
\begin{equation}
\begin{split}
d_{path}({\bf x},{\bf x}')&=
\int\nolimits_0^1{\left|\frac{\partial f(t{\bf x}'+(1-t){\bf x})}{\partial t}\right|dt}\\
&=\int\nolimits_0^1{\left|\sum\nolimits_{j=1}^m{\frac{\partial f_j(t{\bf x}'+(1-t){\bf x})}{\partial t}}\right|dt}\\
&\le\sum\nolimits_{j=1}^m{\int\nolimits_0^1{\left|\frac{\partial f_j(t{\bf x}'+(1-t){\bf x})}{\partial t}\right|dt}}\\
&=\sum\nolimits_{j=1}^m{\left|f_j({\bf x}')-f_j({\bf x})\right|}\\
&=d_{emb}({\bf x},{\bf x}'),
\end{split}
\end{equation}
where we have used the definition of the function $f$ in Eq. \eqref{eq:f} for the second equality and $\left|\sum\nolimits_j{a_j}\right|\le\sum\nolimits_j{\left|a_j\right|}$ for any $a_j\in{\mathbb{R}}$ for the inequality. The second last equality follows from the Lemma \ref{lem:increasing} and the fact that each of the $f_j(t{\bf x}'+(1-t){\bf x})$ is as a softplus function either an increasing or decreasing function in $t$, resulting in the upper bound in Eq. \eqref{eq:bounds}.
\end{proof}

\newpage

\section{Real-World Datasets}\label{sec:data}

The following datasets in Tab. \ref{tab:reg-data} are from the {\sc UCI} repository \cite{kellyuci}:
\begin{description}
\item[{\sc autompg}] The target here is to predict the fuel consumption of cars using their physical properties \cite{r.quinlan1993AutoMPG}. 
\item[{\sc abalone}] The aim here is to predict the age of the Abalone snails using their physical properties \cite{warwicknash1994abalone}. 
\item[{\sc wine-red}] Is a red wine subset of the Wine Quality dataset containing various physicochemical properties of wines \cite{paulocortez2009WineQuality}. 
\item[{\sc wine-white}] Corresponding white wine subset of Wine Quality dataset. 
\item[{\sc concrete}] The objective here is to predict the compressive strength of concrete samples based on their physical and chemical properties \cite{i-chengyeh1998concrete}. 
\item[{\sc air-quality}] Contains hourly averaged air quality measurements spanning approximately one year \cite{devito2008field}. 
\item[{\sc diabetes}] The aim here is to classify the occurrence of diabetes based on diagnostic measurements \cite{kahn00diabetes}. 
\item[{\sc german-credit}] The objective here is to classify people as having good or bad credit risk based on attributes like employment duration, credit history, and loan purpose \cite{hofmann1994statlog}. 
\item[{\sc breast-cancer}] The goal is to classify breast cancer tumors as malignant or benign based on features derived from biopsy samples \cite{matjazzwitter1988breast}. 
\end{description}

The following datasets in Tab. \ref{tab:reg-data} are from the {\sc OpenML} \cite{OpenML2013}:
\begin{description}
\item[{\sc cpu-small}] The aim here is to predict the system activity of a CPU \cite{delve}. 
\item[{\sc higgs-10k}] A subset HIGGS dataset \cite{whiteson2014higgs} used for classifying particle processes as signal or background noise.
\item[{\sc eeg-eye-state}] The target in this dataset is to determine the eye state (open or closed) based on EEG brainwave data \cite{roesler2013eeg}. 
\item[{\sc qm9-10k}] A subset of 10\ 000 heavies molecules from qm9 dataset \cite{ramakrishnan2014Quantum}. 
\item[{\sc superconductor}] Contains information about chemical properties of superconductors \cite{hamidieh2018superconductivty}. 
\end{description}

The following dataset in Tab. \ref{tab:reg-data} was loaded from {\sc scikit-learn} \cite{scikit-learn}:
\begin{description}
\item[{\sc california}] The goal here is to predict the price of houses in California, USA, using geographical characteristics and construction.
\end{description}

\clearpage
\newpage

\section{Supervised Learning Features Results}\label{sec:appendix_features}

Tab. \ref{tab:features_all} presents supplementary results to Sect. \ref{sec:emb}.

\begin{table}[h]
\caption{Feature creation experiment results. Top: $R^2$, bottom: accuracy, standard deviation on parenthesis. Higher values are better. Note that {\sc ivis} crashed on multiple occasions on the datasets with many data points, and {\sc lol} requires classification labels. Results are missing for {\sc pca} and {\sc lol} when $p < m$ as it is not possible to select more components $m$ than the data covariates $p$}
\label{tab:features_all}
\begin{center}
\begin{small}
\begin{sc}
\begin{tabular}{llllll}
\toprule
dataset & $m$ & gbmap & pca & ivis & lol \\
\midrule
california & 2 & 0.64 ($\pm$ 0.01) & 0.03 ($\pm$ 0.01) & \textbf{0.72 ($\pm$ 0.01)} & $-$ \\
california & 8 & 0.72 ($\pm$ 0.01) & 0.61 ($\pm$ 0.01) & \textbf{0.77 ($\pm$ 0.01)} & $-$ \\
california & 12 & 0.73 ($\pm$ 0.01) & $p < m$ & \textbf{0.77 ($\pm$ 0.01)} & $-$ \\
california & 16 & 0.73 ($\pm$ 0.01) & $p < m$ & \textbf{0.77 ($\pm$ 0.01)} & $-$ \\
california & 32 & 0.74 ($\pm$ 0.01) & $p < m$ & \textbf{0.78 ($\pm$ 0.01)} & $-$ \\
concrete & 2 & \textbf{0.73 ($\pm$ 0.04)} & 0.11 ($\pm$ 0.06) & crash & $-$ \\
concrete & 8 & \textbf{0.85 ($\pm$ 0.01)} & 0.62 ($\pm$ 0.03) & 0.76 ($\pm$ 0.03) & $-$ \\
concrete & 12 & \textbf{0.85 ($\pm$ 0.02)} & $p < m$ & 0.78 ($\pm$ 0.03) & $-$ \\
concrete & 16 & \textbf{0.86 ($\pm$ 0.02)} & $p < m$ & 0.78 ($\pm$ 0.04) & $-$ \\
concrete & 32 & \textbf{0.85 ($\pm$ 0.05)} & $p < m$ & 0.78 ($\pm$ 0.03) & $-$ \\
cpu-small & 2 & \textbf{0.97 ($\pm$ 0)} & 0.36 ($\pm$ 0.02) & 0.94 ($\pm$ 0.02) & $-$ \\
cpu-small & 8 & \textbf{0.97 ($\pm$ 0)} & 0.69 ($\pm$ 0.01) & \textbf{0.97 ($\pm$ 0)} & $-$ \\
cpu-small & 12 & \textbf{0.97 ($\pm$ 0)} & 0.72 ($\pm$ 0.01) & \textbf{0.97 ($\pm$ 0)} & $-$ \\
cpu-small & 16 & \textbf{0.97 ($\pm$ 0.01)} & $p < m$ & \textbf{0.97 ($\pm$ 0)} & $-$ \\
cpu-small & 32 & \textbf{0.97 ($\pm$ 0.01)} & $p < m$ & \textbf{0.97 ($\pm$ 0)} & $-$ \\
superconductor & 2 & 0.82 ($\pm$ 0) & 0.46 ($\pm$ 0.01) & \textbf{0.83 ($\pm$ 0.01)} & $-$ \\
superconductor & 8 & 0.84 ($\pm$ 0) & 0.56 ($\pm$ 0) & \textbf{0.85 ($\pm$ 0.01)} & $-$ \\
superconductor & 12 & 0.84 ($\pm$ 0.01) & 0.59 ($\pm$ 0) & \textbf{0.85 ($\pm$ 0.01)} & $-$ \\
superconductor & 16 & 0.84 ($\pm$ 0) & 0.60 ($\pm$ 0) & \textbf{0.86 ($\pm$ 0.01)} & $-$ \\
superconductor & 32 & \textbf{0.85 ($\pm$ 0)} & 0.69 ($\pm$ 0.01) & \textbf{0.85 ($\pm$ 0.01)} & $-$ \\
qm9-10k & 2 & \textbf{0.51 ($\pm$ 0.01)} & 0.07 ($\pm$ 0) & 0.49 ($\pm$ 0.09) & $-$ \\
qm9-10k & 8 & \textbf{0.61 ($\pm$ 0.02)} & 0.19 ($\pm$ 0.01) & 0.57 ($\pm$ 0.04) & $-$ \\
qm9-10k & 12 & \textbf{0.65 ($\pm$ 0.02)} & 0.27 ($\pm$ 0.01) & 0.57 ($\pm$ 0.03) & $-$ \\
qm9-10k & 16 & \textbf{0.66 ($\pm$ 0.02)} & 0.34 ($\pm$ 0.01) & 0.58 ($\pm$ 0.03) & $-$ \\
qm9-10k & 32 & \textbf{0.68 ($\pm$ 0.02)} & $p < m$ & 0.59 ($\pm$ 0.03) & $-$ \\
synth-cos-r & 2 & \textbf{0.06 ($\pm$ 0)} & 0 ($\pm$ 0) & crash & $-$ \\
synth-cos-r & 8 & \textbf{0.16 ($\pm$ 0)} & 0 ($\pm$ 0) & crash & $-$ \\
synth-cos-r & 12 & \textbf{0.20 ($\pm$ 0)} & 0 ($\pm$ 0) & crash & $-$ \\
synth-cos-r & 16 & 0.23 ($\pm$ 0) & 0 ($\pm$ 0) & \textbf{0.64 ($\pm$ 0.01)} & $-$ \\
synth-cos-r & 32 & \textbf{0.34 ($\pm$ 0)} & 0 ($\pm$ 0) & crash & $-$ \\
synth-cos-r & 64 & 0.46 ($\pm$ 0) & 0 ($\pm$ 0) & \textbf{0.66 ($\pm$ 0.01)} & $-$ \\
synth-cos-r & 128 & \textbf{0.52 ($\pm$ 0)} & 0 ($\pm$ 0) & crash & $-$ \\
\midrule
higgs-10k & 2 & \textbf{0.66 ($\pm$ 0.01)} & 0.53 ($\pm$ 0.01) & 0.61 ($\pm$ 0.02) & 0.59 ($\pm$ 0.01) \\
higgs-10k & 8 & \textbf{0.68 ($\pm$ 0.01)} & 0.54 ($\pm$ 0.01) & 0.67 ($\pm$ 0.01) & 0.61 ($\pm$ 0.01) \\
higgs-10k & 12 & \textbf{0.68 ($\pm$ 0.01)} & 0.54 ($\pm$ 0) & \textbf{0.68 ($\pm$ 0.01)} & 0.61 ($\pm$ 0.01) \\
higgs-10k & 16 & \textbf{0.68 ($\pm$ 0.01)} & 0.55 ($\pm$ 0) & \textbf{0.68 ($\pm$ 0.01)} & 0.61 ($\pm$ 0.01) \\
higgs-10k & 32 & 0.67 ($\pm$ 0.01) & $p < m$ & \textbf{0.68 ($\pm$ 0.01)} & $p < m$ \\
eeg-eye-state & 2 & \textbf{0.61 ($\pm$ 0.01)} & 0.55 ($\pm$ 0) & 0.60 ($\pm$ 0.01) & 0.57 ($\pm$ 0) \\
eeg-eye-state & 8 & 0.65 ($\pm$ 0.01) & 0.57 ($\pm$ 0.01) & \textbf{0.80 ($\pm$ 0.01)} & 0.58 ($\pm$ 0.01) \\
eeg-eye-state & 12 & 0.69 ($\pm$ 0.02) & 0.62 ($\pm$ 0.01) & \textbf{0.80 ($\pm$ 0.01)} & 0.61 ($\pm$ 0.01) \\
eeg-eye-state & 16 & 0.71 ($\pm$ 0.01) & $p < m$ & \textbf{0.81 ($\pm$ 0.01)} & $p < m$ \\
eeg-eye-state & 32 & 0.75 ($\pm$ 0.01) & $p < m$ & \textbf{0.81 ($\pm$ 0.01)} & $p < m$ \\
synth-cos-c & 2 & \textbf{0.56 ($\pm$ 0)} & 0.50 ($\pm$ 0) & crash & 0.50 ($\pm$ 0) \\
synth-cos-c & 8 & \textbf{0.59 ($\pm$ 0)} & 0.50 ($\pm$ 0) & crash & 0.50 ($\pm$ 0) \\
synth-cos-c & 12 & \textbf{0.60 ($\pm$ 0)} & 0.50 ($\pm$ 0) & crash & 0.50 ($\pm$ 0) \\
synth-cos-c & 16 & 0.61 ($\pm$ 0) & 0.50 ($\pm$ 0) & \textbf{0.67 ($\pm$ 0)} & 0.50 ($\pm$ 0) \\
synth-cos-c & 32 & \textbf{0.62 ($\pm$ 0)} & 0.50 ($\pm$ 0) & crash & 0.50 ($\pm$ 0) \\
synth-cos-c & 64 & 0.64 ($\pm$ 0) & 0.50 ($\pm$ 0) & \textbf{0.68 ($\pm$ 0)} & 0.50 ($\pm$ 0) \\
synth-cos-c & 128 & \textbf{0.64 ($\pm$ 0)} & 0.50 ($\pm$ 0) & crash & 0.50 ($\pm$ 0) \\
\end{tabular}
\end{sc}
\end{small}
\end{center}
\end{table}

\clearpage
\newpage

\section{Out-of-Distribution Detection}

\label{sec:drift}

\subsection{Regression Datasets}

Fig. \ref{fig:gbmapdrifterregall} and Fig. \ref{fig:euclideandrifterregall} show the results of the {\sc gbmap} and the {\sc euclid} drifters on all considered regression datasets. There is a clear correlation between the {\sc gbmap} drift indicator and the actual loss, explaining the high AUC values (around $0.9$ for most datasets) reported in Tab. \ref{tab:reg} (the last block of columns). Meanwhile, the {\sc euclid} indicator is inferior for most datasets.

\begin{figure}
     \centering
     \subfloat[{\sc autompg}]{
	\includegraphics[width=0.45\textwidth]{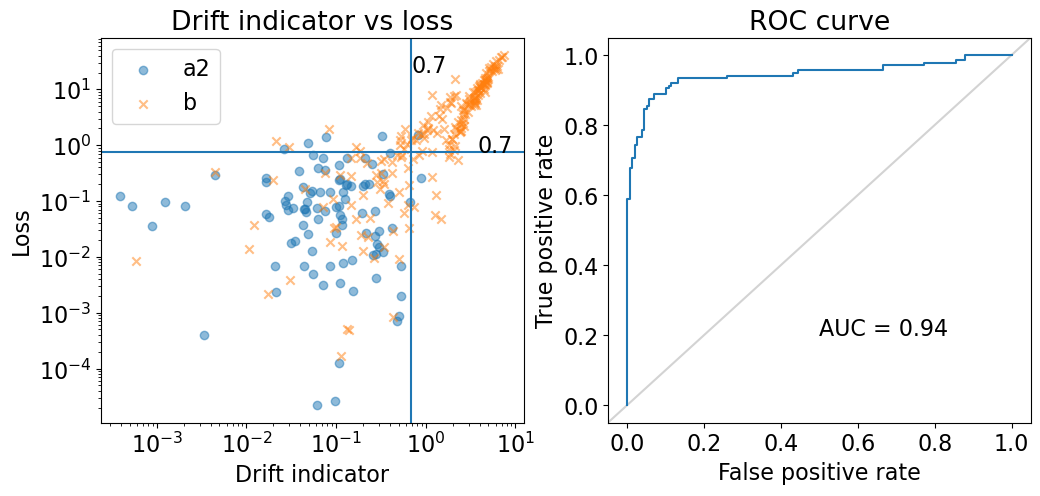} 
	}
    \subfloat[{\sc abalone}]{
	\includegraphics[width=0.45\textwidth]{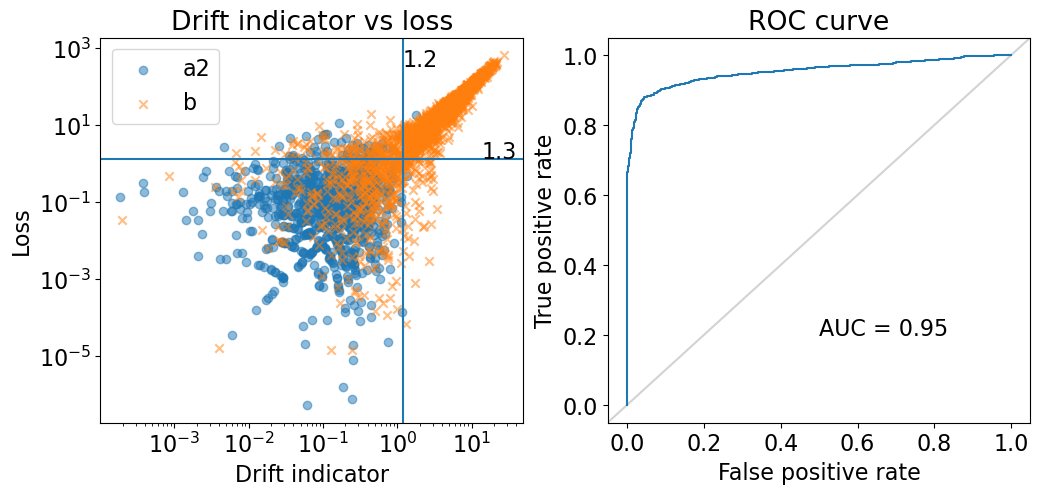} 
	}
 \newline
\noindent 
    \subfloat[{\sc california}]{
	\includegraphics[width=0.45\textwidth]{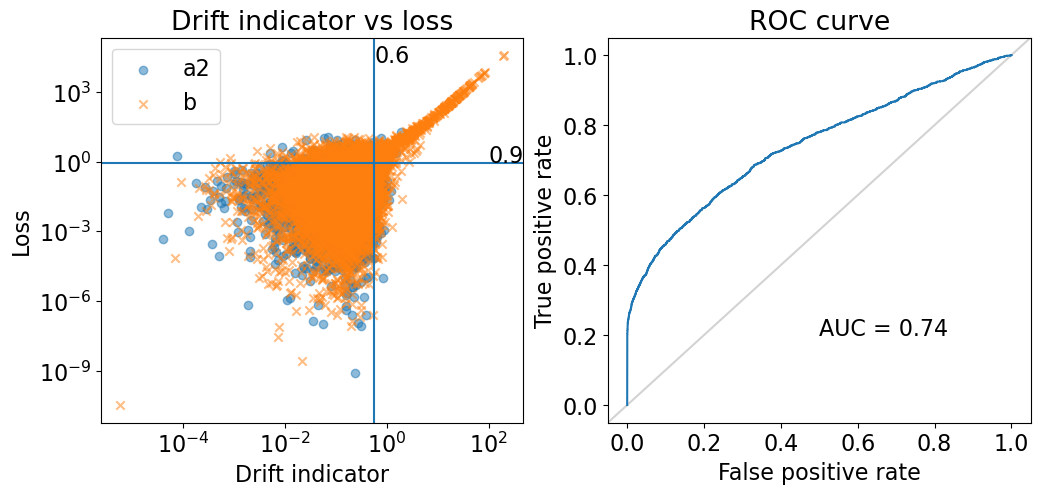} 
	}
    \subfloat[{\sc wine-red}]{
	\includegraphics[width=0.45\textwidth]{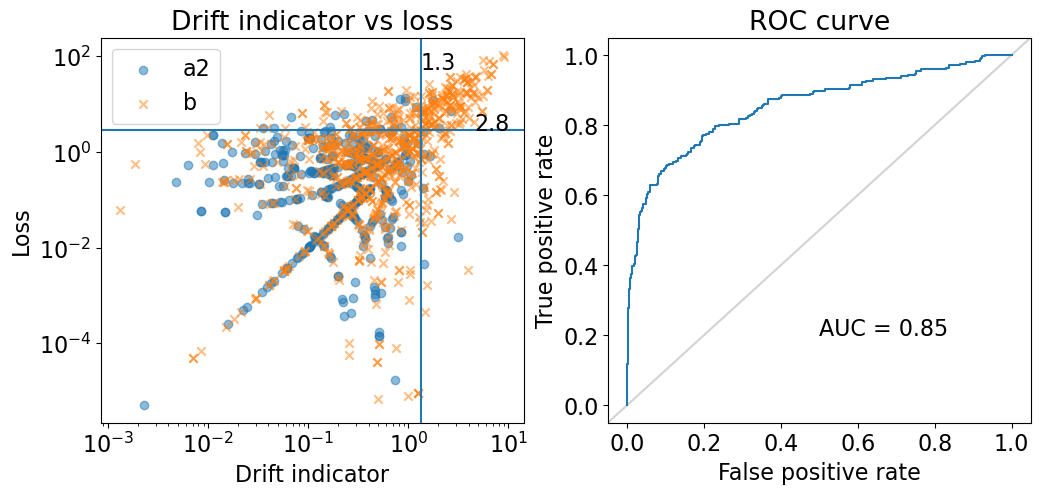} 
	}
 \newline
\noindent 
    \subfloat[{\sc wine-white}]{
	\includegraphics[width=0.45\textwidth]{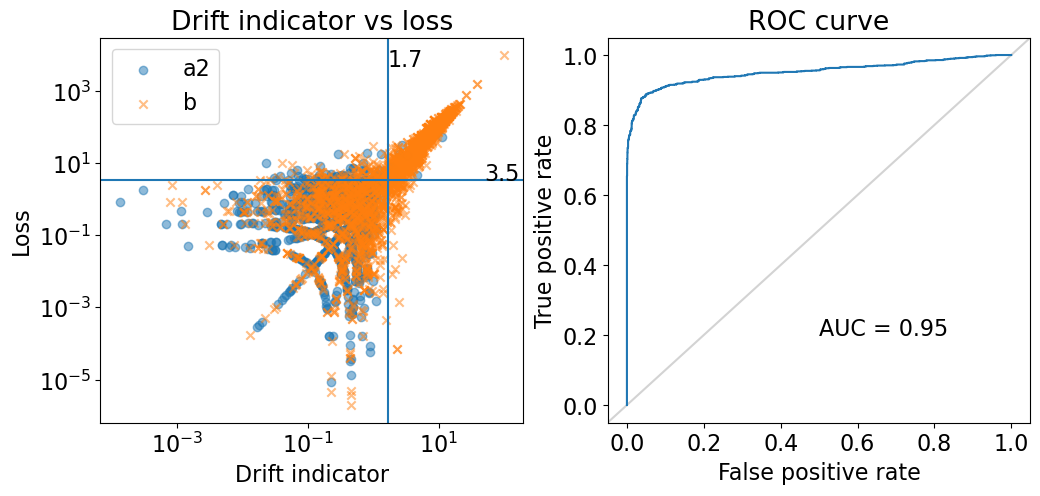} 
	}
    \subfloat[{\sc concrete}]{
	\includegraphics[width=0.45\textwidth]{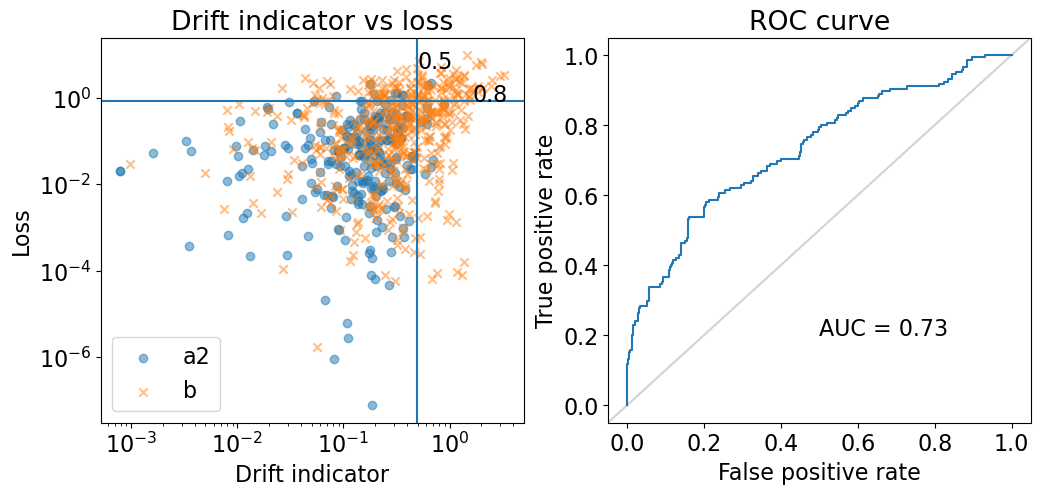} 
	}
 \newline
\noindent 
    \subfloat[{\sc cpu-small}]{
	\includegraphics[width=0.45\textwidth]{figs/drift/cpu-small_gbmap_0_1.0.png} 
	}
     \subfloat[{\sc airquality}]{
	\includegraphics[width=0.45\textwidth]{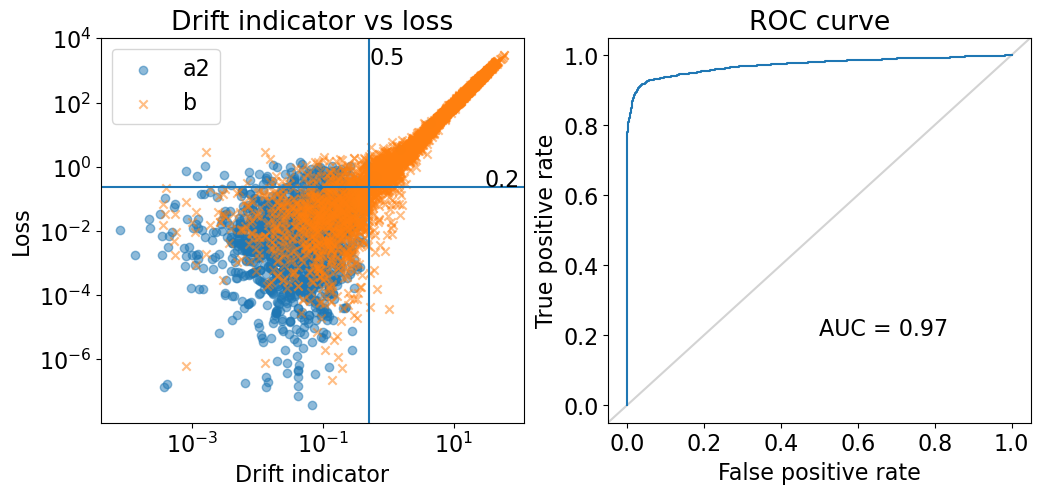} 
	}
 \newline
\noindent 
     \subfloat[{\sc qm9-10k}]{
	\includegraphics[width=0.45\textwidth]{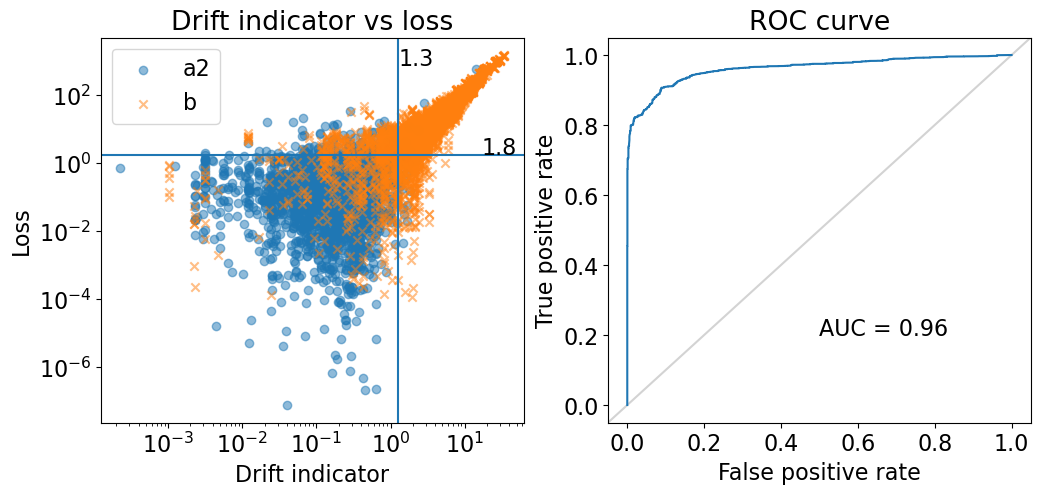} 
	}
    \subfloat[{\sc superconductor}]{
	\includegraphics[width=0.45\textwidth]{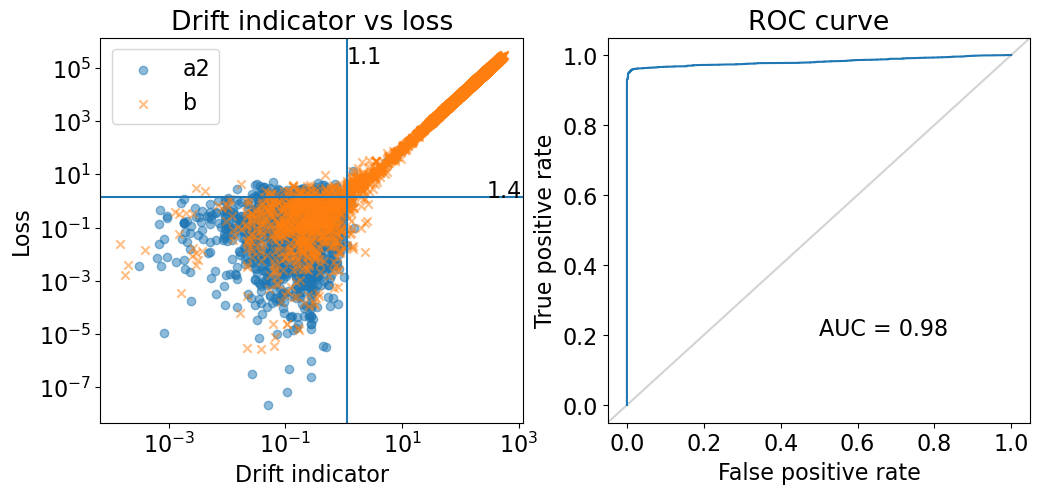} 
	}

 \caption{{\sc gbmap} drifter for regression datasets}
 \label{fig:gbmapdrifterregall}
\end{figure}

\begin{figure}
     \centering
    \subfloat[{\sc autompg}]{
	\includegraphics[width=0.45\textwidth]{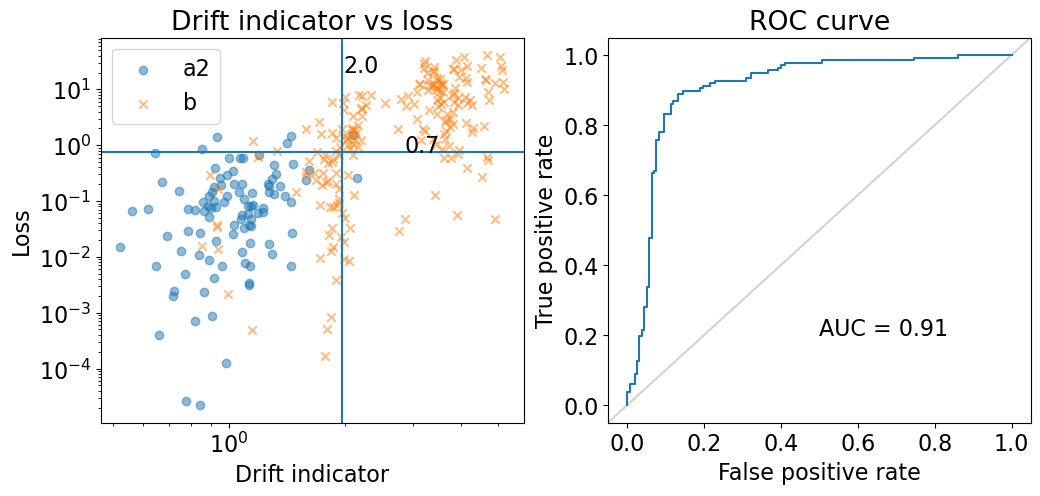} 
	}
    \subfloat[{\sc abalone}]{
	\includegraphics[width=0.45\textwidth]{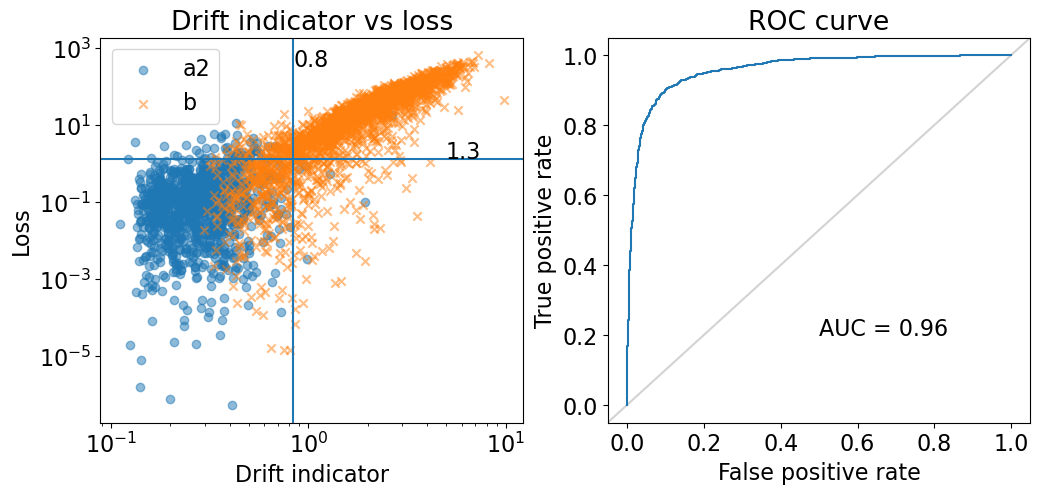} 
	}
 \newline
\noindent 
    \subfloat[{\sc california}]{
	\includegraphics[width=0.45\textwidth]{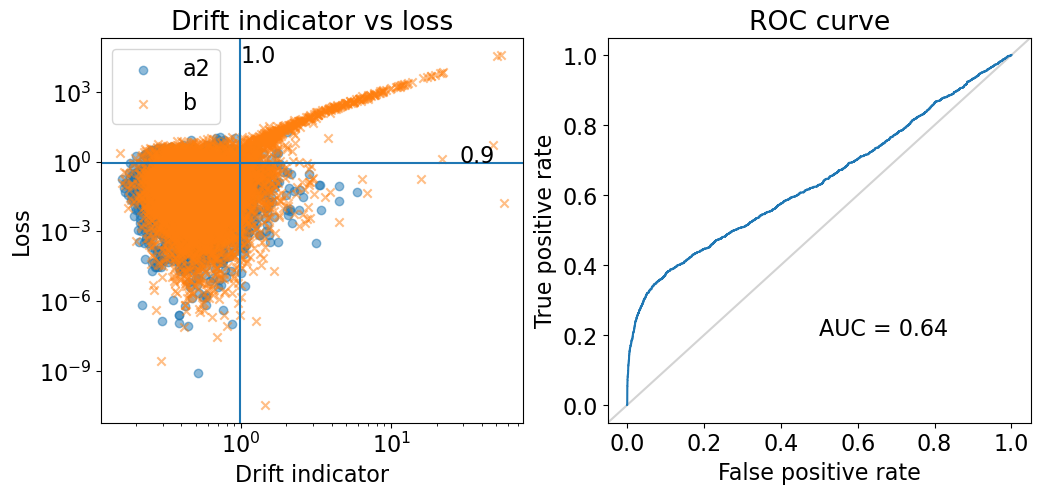} 
	}
    \subfloat[{\sc wine-red}]{
	\includegraphics[width=0.45\textwidth]{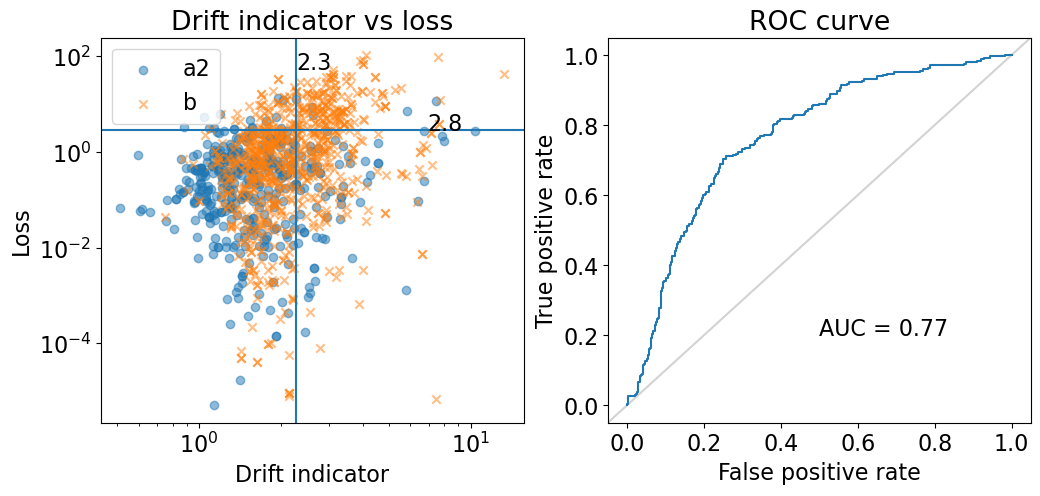} 
	}
\newline
\noindent 
    \subfloat[{\sc wine-white}]{
	\includegraphics[width=0.45\textwidth]{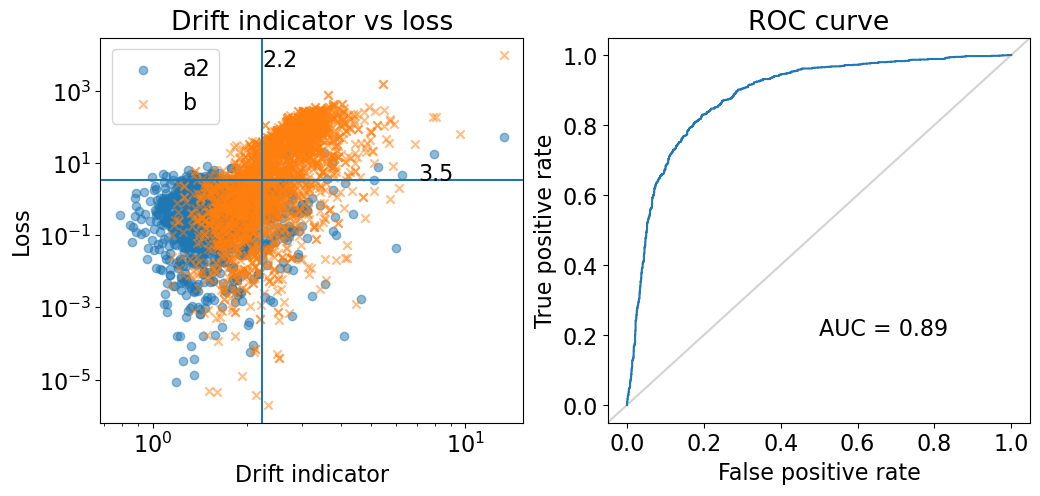} 
	}
    \subfloat[{\sc concrete}]{
	\includegraphics[width=0.45\textwidth]{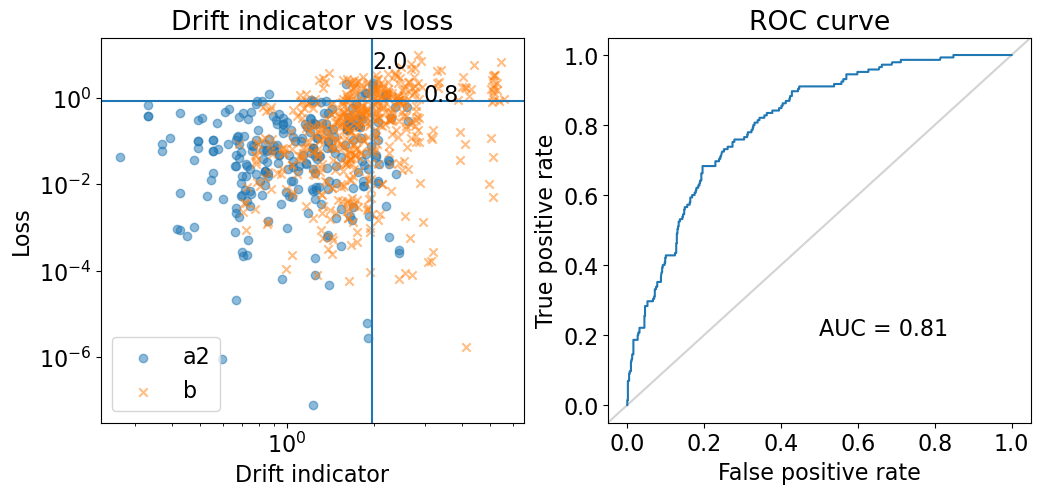} 
	}
 \newline
\noindent 
    \subfloat[{\sc cpu-small}]{
	\includegraphics[width=0.45\textwidth]{figs/drift/cpu-small_euclidean_0_1.0.png} 
	}
    \subfloat[{\sc airquality}]{
	\includegraphics[width=0.45\textwidth]{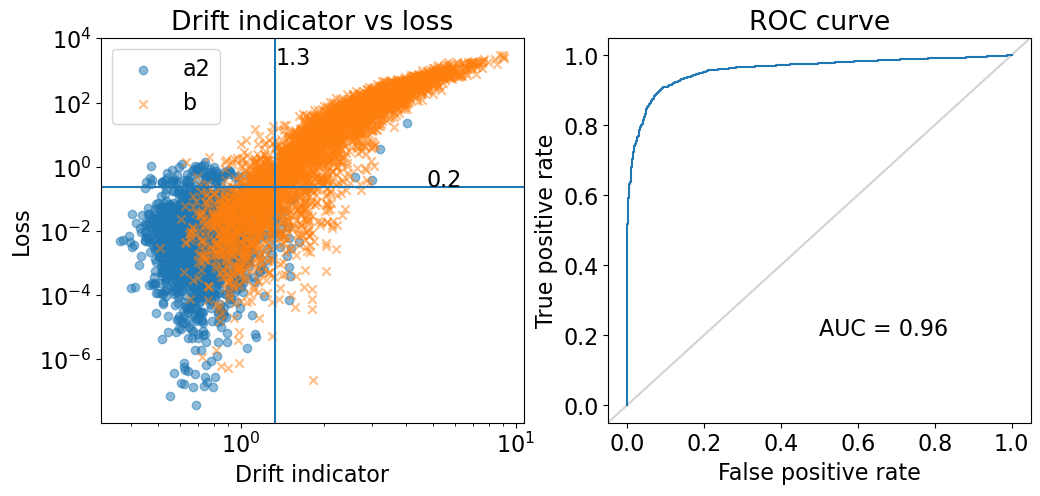} 
	}
 \newline
\noindent 
    \subfloat[{\sc qm9-10k}]{
	\includegraphics[width=0.45\textwidth]{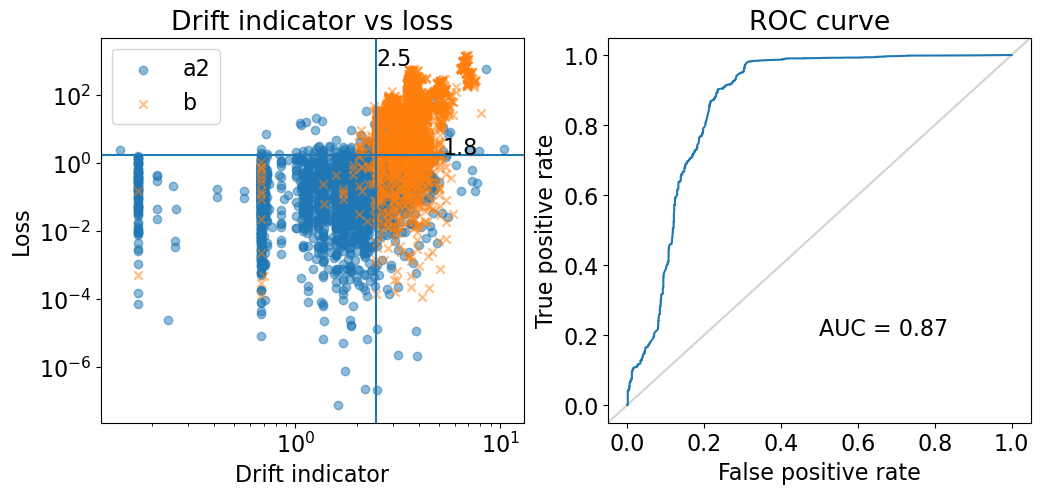} 
	}
    \subfloat[{\sc superconductor}]{
	\includegraphics[width=0.45\textwidth]{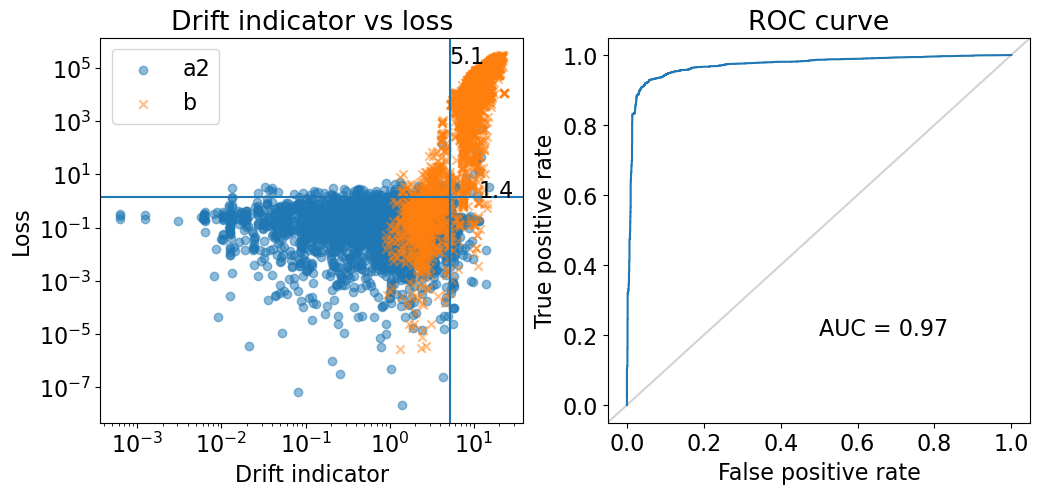} 
	}

 \caption{{\sc euclid} drifter for regression datasets}
 \label{fig:euclideandrifterregall}
\end{figure}

\clearpage
\newpage

\subsection{Classification Datasets}

Fig. \ref{fig:gbmapdrifterclsall} and Fig. \ref{fig:euclideandrifterclsall} show the results of the {\sc gbmap} drifter and the {\sc euclid} drifter on all considered classification datasets. The drifters perform similarly for the classification datasets.

\begin{figure}[h]
     \centering
     \subfloat[{\sc breast-cancer}]{
	\includegraphics[width=0.45\textwidth]{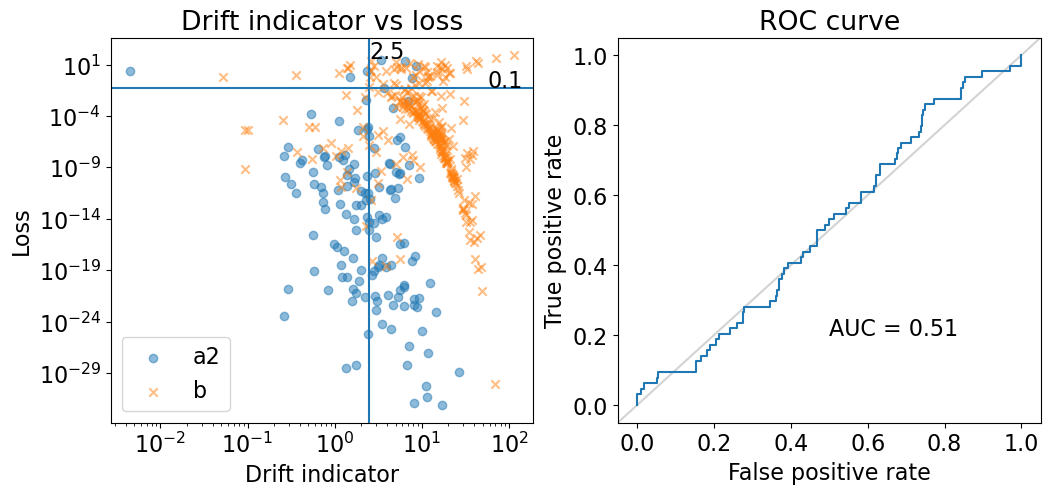} 
	}
 \subfloat[{\sc diabetes}]{
	\includegraphics[width=0.45\textwidth]{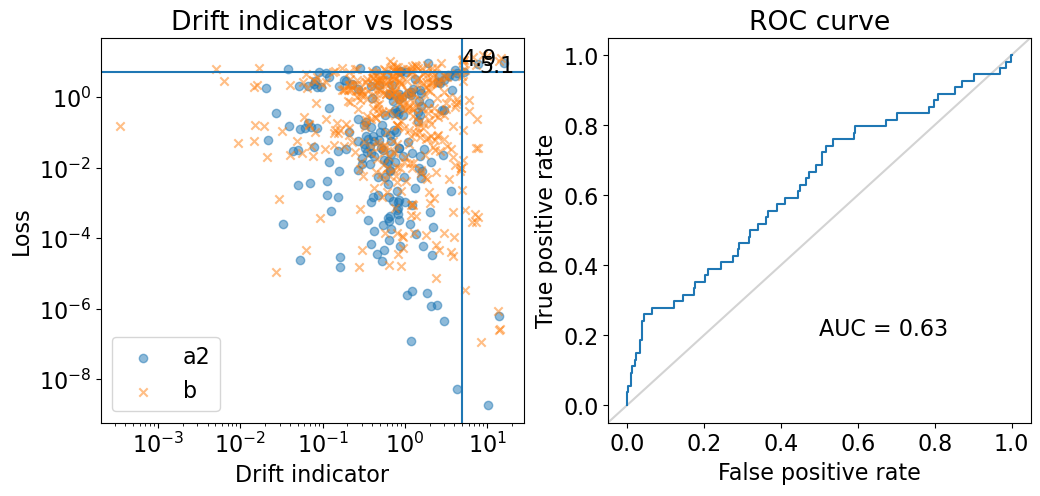} 
	}
 \newline
\noindent 
 \subfloat[{\sc german-credit}]{
	\includegraphics[width=0.45\textwidth]{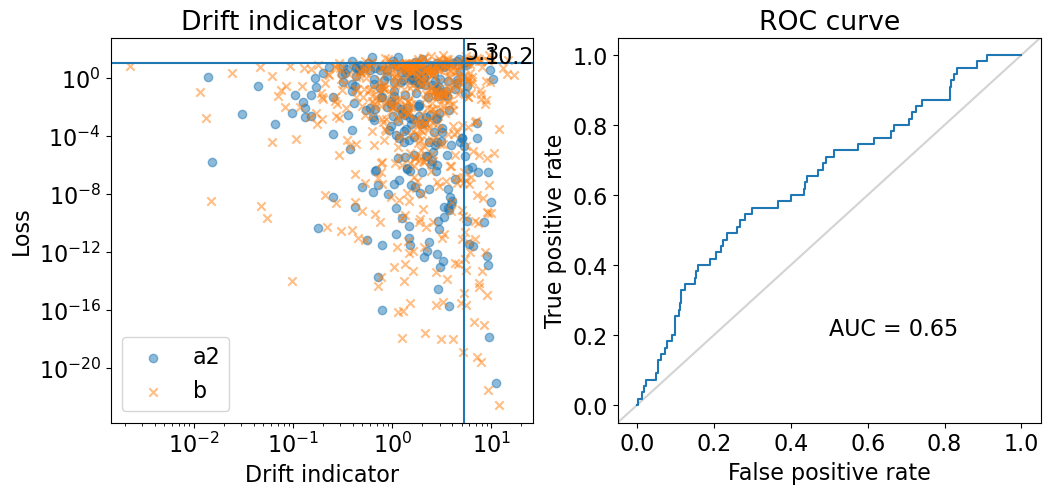} 
	}
 \subfloat[{\sc higgs-10k}]{
	\includegraphics[width=0.45\textwidth]{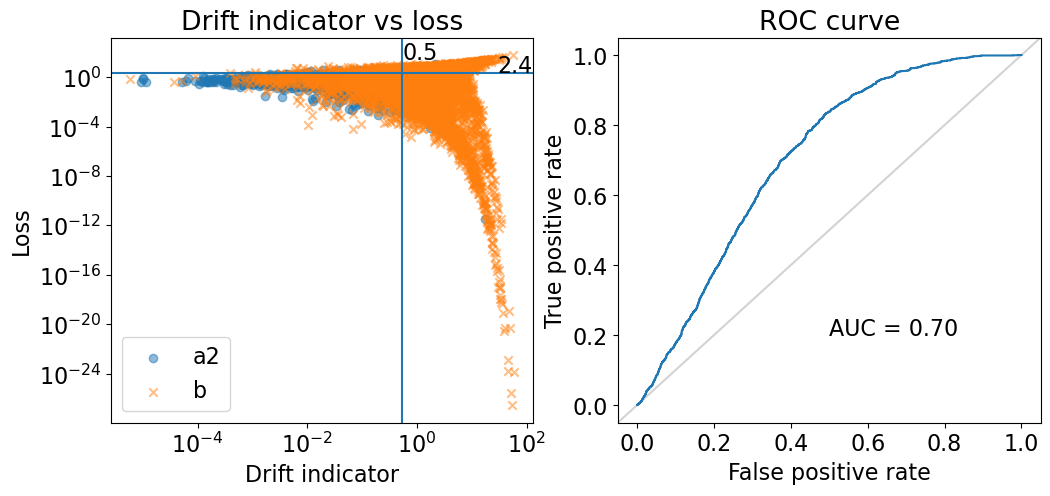} 
	}
 \newline
\noindent 
 \subfloat[{\sc eeg-eye-state}]{
	\includegraphics[width=0.45\textwidth]{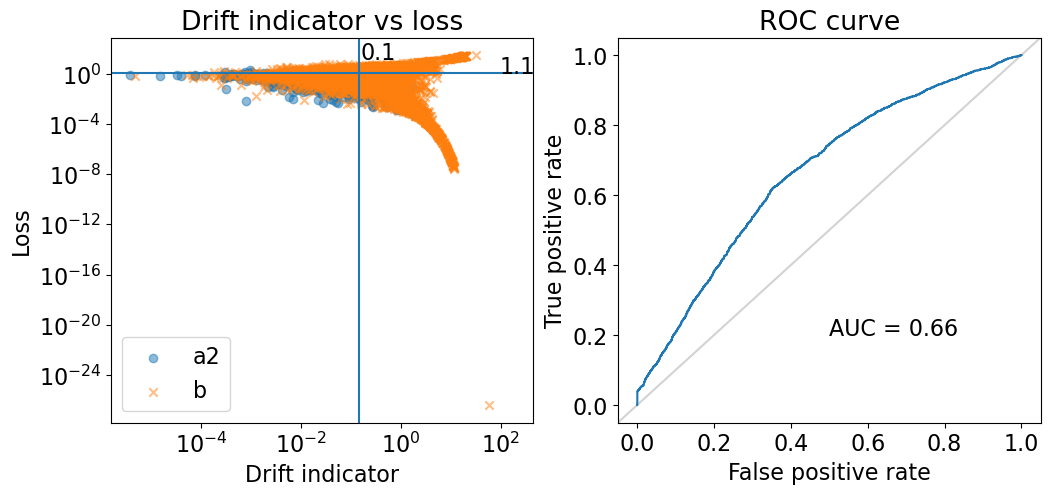} 
	}

 \caption{{\sc gbmap} drifter for classification datasets}
 \label{fig:gbmapdrifterclsall}
\end{figure}

\begin{figure}
     \centering
   \subfloat[{\sc breast-cancer}]{
	\includegraphics[width=0.45\textwidth]{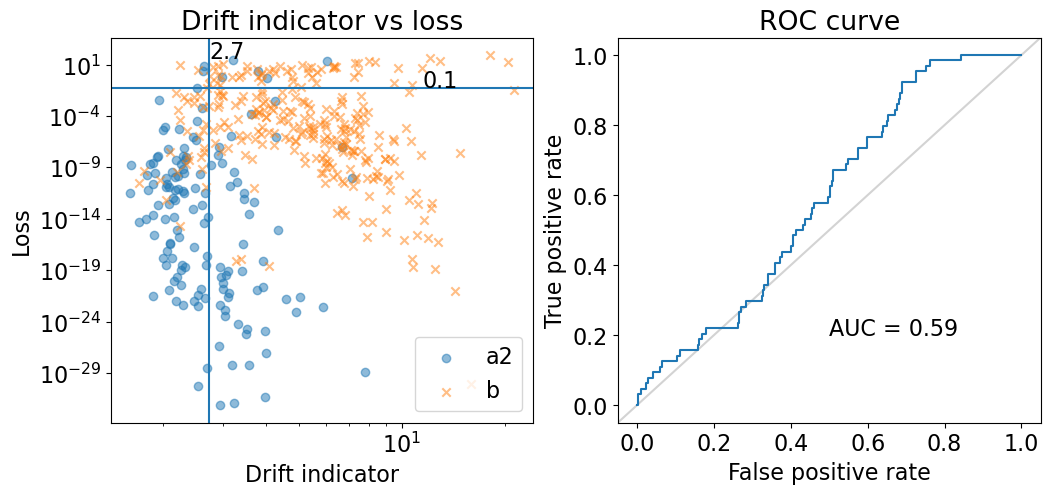} 
	}
    \subfloat[{\sc diabetes}]{
	\includegraphics[width=0.45\textwidth]{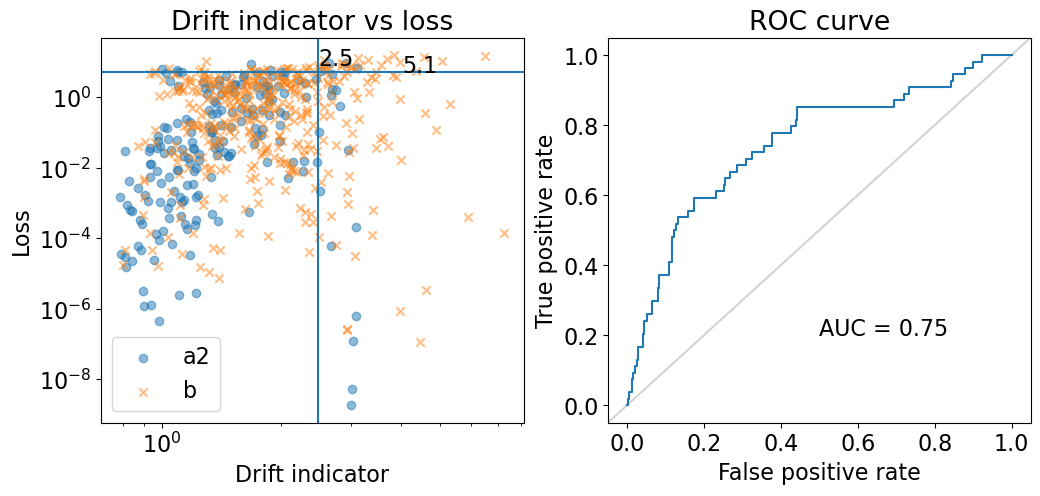} 
	}
\newline
\noindent 
 \subfloat[{\sc german-credit}]{
	\includegraphics[width=0.45\textwidth]{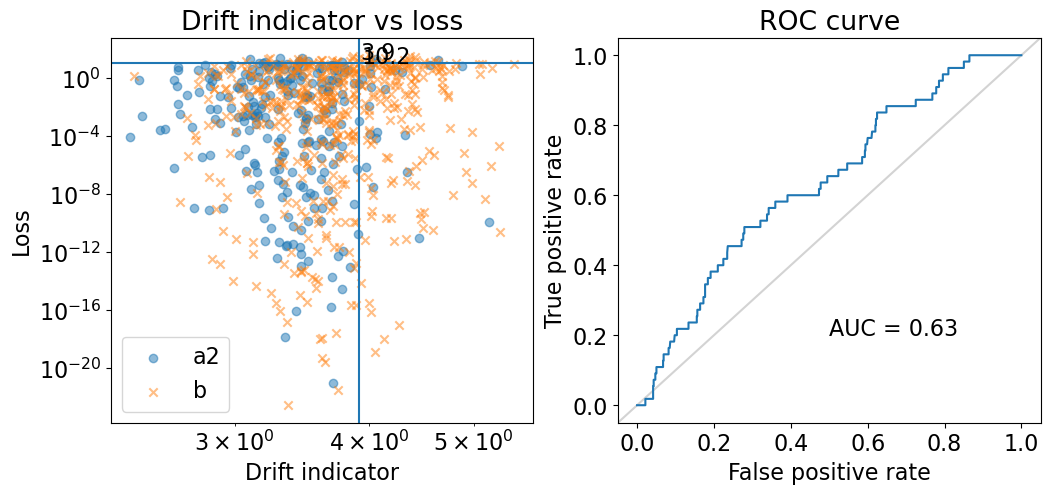} 
	}
    \subfloat[{\sc higgs-10k}]{
	\includegraphics[width=0.45\textwidth]{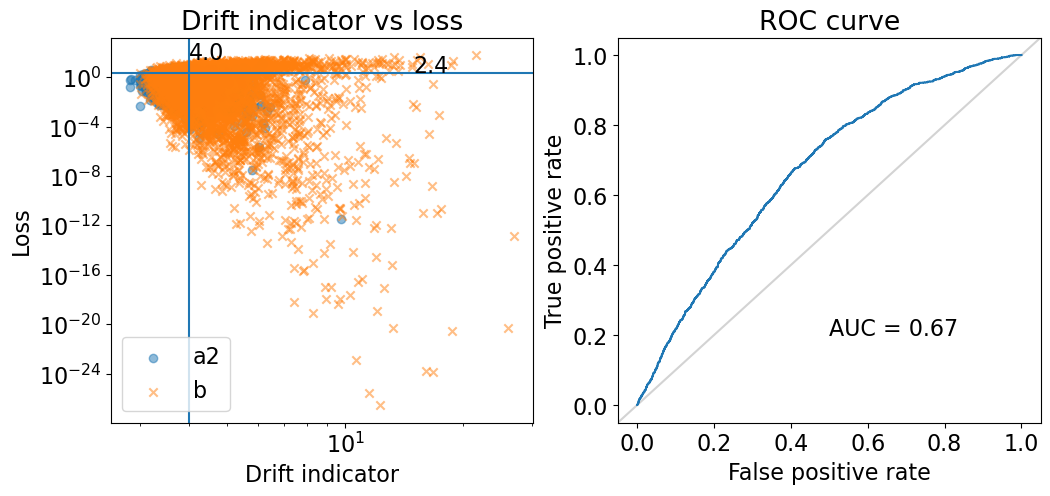} 
	}
\newline
\noindent 
 \subfloat[{\sc eeg-eye-state}]{
	\includegraphics[width=0.45\textwidth]{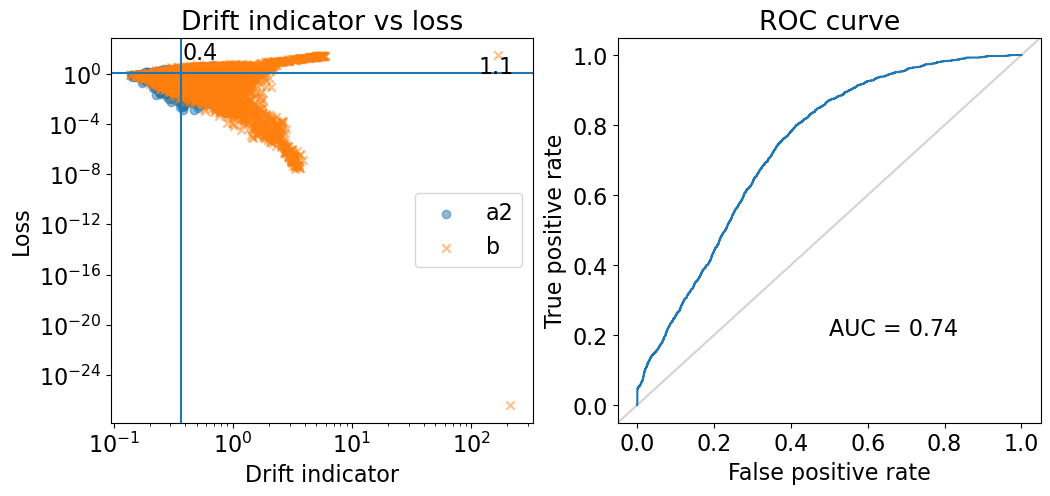} 
	}
 \caption{{\sc euclid} drifter for classification datasets}
 \label{fig:euclideandrifterclsall}
\end{figure}

\clearpage
\newpage
\section{Visualizations}\label{sec:vis2}

\begin{figure}[h]
    \centering
    \includegraphics[width=0.4\textwidth]{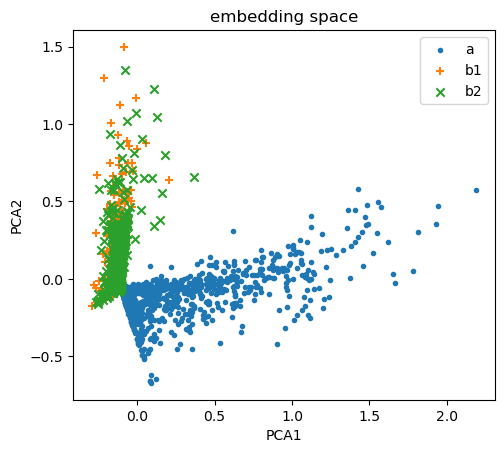}
    \includegraphics[width=0.4\textwidth]{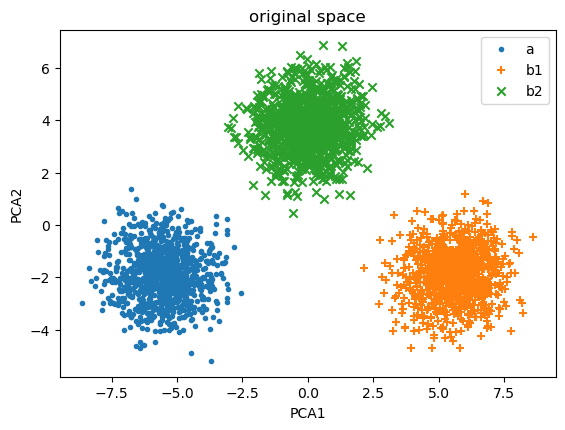}
    \caption{PCA of the embedding (left) and the original space (right).}
    \label{fig:vis}
\end{figure}

Here, we show an example of a visualization of the embedding space by PCA and compare it to the PCA visualization of the original data space, demonstrating some properties of the embedding and the embedding distance.

Consider an 8-dimensional (plus the intercept term) synthetic regression dataset, similar to {\sc synth-cos-r}, where there are three equal-sized clusters ({\tt a}, {\tt b1}, {\tt b2}) in the covariate space and where dimensions 1--4 are relevant for the regression task, and dimensions 5--8 are irrelevant for the regression task. Further, assume that the cluster {\tt a} differs from cluster {\tt b2} in the relevant dimensions 1--4 and the cluster {\tt b1} differs from {\tt b2} only in the irrelevant dimensions 5--8. We expect an unsupervised dimensionality reduction method, such as PCA, to show three clusters. However, if the PCA is done on the embedding space, clusters {\tt b1} and {\tt b2} should merge since their difference is irrelevant for the regression task and the irrelevant directions should be ignored by the embedding, which is what happens; see the visualization in  Fig. \ref{fig:vis}.

Specifically, the data matrix ${\bf X}\in{\mathbb{R}}^{3000\times 9}$ is constructed as follows, where ${\bf N}_{ij}$ are independent draws from a normal distribution with zero mean and unit variance:
\begin{equation}
{\bf X}_{ij}=\left\{
\begin{array}{lcl}
{\bf N}_{ij}+4 & , & i\in\{1,\ldots,1000\}\wedge j\in\{1,2,3,4\}\\
{\bf N}_{ij}+4 & , & i\in\{1001,\ldots,2000\}\wedge j\in\{5,6,7,8\}\\
1 &,& j=9\\
{\bf N}_{ij}&,&{\rm otherwise}
\end{array}
\right. 
\end{equation}
The regression target vector ${\bf y}\in{\mathbb{R}}^{3000}$ is given by
\begin{equation} 
{\bf y}=\cos{\left({\bf X}\right)}{\bf u},
\end{equation}
where ${\bf u}\in{\mathbb{R}}^9$ is a random unit vector satisfying ${\bf u}_5={\bf u}_6={\bf u}_7={\bf u}_8={\bf u}_9=0$ and $\cos()$ denotes element-wise cosine.

Cluster {\tt a} is composed of points ${\bf X}_{i\cdot}$ where $i\in\{1,\ldots,1000\}$, cluster {\tt b1} is composed of points ${\bf X}_{i\cdot}$ where $i\in\{1001,\ldots,2000\}$, and cluster {\tt b2} is composed of points ${\bf X}_{i\cdot}$ where $i\in\{2001,\ldots,3000\}$.

We have used the parameters $m=10$, $\lambda=10^{-3}$, and $\beta=5$ for {\sc gbmap}.

\end{document}